\documentclass[reqno, 10pt]{article}%

\RequirePackage{amsthm, amsmath, natbib, amsfonts, amssymb}%

\usepackage{graphicx, color}%
\usepackage{tikz}%
\usepackage{natbib}%
\usepackage{mathrsfs}

\usepackage{tikz}

\usepackage[boxruled]{algorithm2e}

\newcommand \cB{{\cal B}}
\newcommand \cC{{\cal C}}

\newcommand \cF{{\cal F}}

\newcommand \cL{{\cal L}}

\newcommand \cN{{\cal N}}

\newcommand \R{{\mathbb  R}}
\newcommand \E{{\mathbb  E}}

\newcommand{\norm}[1]{\|#1\|}%

\newcommand{\bs}{\boldsymbol}

\newcommand{\FF}{\hat {\mathcal F}}

\renewcommand{\P}{\mathbb P}%

\DeclareMathOperator{\conv}{conv}
\DeclareMathOperator{\seg}{seg}
\DeclareMathOperator{\Star}{star}
\DeclareMathOperator{\diam}{diam}

\DeclareMathOperator*{\argmin}{argmin}

\DeclareMathOperator{\pen}{pen}

\newcommand{\1}{{\rm 1}\kern-0.24em{\rm I}}

\renewcommand{\hat}{\widehat}

\newtheorem{theorem}{Theorem}%
\newtheorem{corollary}{Corollary}%
\newtheorem{lemma}{Lemma}%
\theoremstyle{assumption}%
\newtheorem{assumption}{Assumption}{\bf}{\rm}%
\newtheorem{definition}{Definition}%
\newtheorem{question}{Question}%
\theoremstyle{remark}%
\newtheorem{remark}{Remark}%

\begin{document}

\newlength{\figwidth}
\newlength{\figheight}

\title{Hyper-sparse optimal aggregation}%

\author{St\'ephane Ga\"iffas$^{1, 3}$ and Guillaume Lecu\'e$^{2, 3}$}

\footnotetext[1]{Universit\'e Pierre et Marie Curie - Paris~6,
  Laboratoire de Statistique Th\'eorique et Appliqu\'ee. \emph{email}:
  \texttt{stephane.gaiffas@upmc.fr}}

\footnotetext[2] {CNRS, Laboratoire d'Analyse et Math\'ematiques
  appliqu\'ees, Universit\'e Paris-Est - Marne-la-vall\'ee
  \emph{email}: guillaume.lecue@univ-mlv.fr}

\footnotetext[3]{This work is supported by French Agence Nationale de
  la Recherce (ANR) ANR Grant \textsc{``Prognostic''}
  ANR-09-JCJC-0101-01. (\texttt{http://www.lsta.upmc.fr/prognostic/index.php})
}

\maketitle

  \begin{abstract}
    In this paper, we consider the problem of \emph{hyper-sparse
      aggregation}. Namely, given a dictionary $F = \{ f_1, \ldots,
    f_M \}$ of functions, we look for an optimal aggregation algorithm
    that writes $\tilde f = \sum_{j=1}^M \theta_j f_j$ with as many
    zero coefficients $\theta_j$ as possible. This problem is of
    particular interest when $F$ contains many irrelevant functions
    that should not appear in $\tilde{f}$. We provide an exact oracle
    inequality for $\tilde f$, where only two coefficients are
    non-zero, that entails $\tilde f$ to be an optimal aggregation
    algorithm.  Since selectors are suboptimal aggregation procedures,
    this proves that $2$ is the minimal number of elements of $F$
    required for the construction of an optimal aggregation procedures
    in every situations. A simulated example of this algorithm is
    proposed on a dictionary obtained using LARS, for the problem of
    selection of the regularization parameter of the LASSO. We also give an
    example of use of aggregation to achieve minimax adaptation over
    anisotropic Besov spaces, which was not previously known in
    minimax theory (in regression on a random design). \\
    
    \noindent%
  \end{abstract}

\section{Introduction}
\label{sec:introduction}

\subsection{Motivations}

In this paper, we consider the problem of \emph{sparse aggregation}.
Namely, given a dictionary $F = \{ f_1, \ldots, f_M \}$ of functions,
we look for an optimal aggregation algorithm that writes $\hat f =
\sum_{j=1}^M \theta_j f_j$ with as many zero coefficients $\theta_j$
as possible. This question appears when one wants to use aggregation
procedures to construct adaptive procedures. Indeed, in practice many
elements of the dictionary appear to be irrelevant.  We
would like to remove completely these irrelevant elements of the
dictionary from the final aggregate, while keeping the optimality of
the procedure (``optimality'' is used in reference to the definition
of ``optimal aggregation procedure'' provided in
\cite{tsy:03} and \cite{LM06}). Moreover, one could imagine large dictionaries
containing many different types of estimators (kernel estimators,
projection estimators, etc.) with many different parameters (smoothing
parameters, groups of variables, etc.). Some of the estimators are
likely to be more adapted than the others, depending on the kind of
models that fits well to the data. So, we would like to construct
procedures that can adapt to different models combining only the
estimators, contained in the dictionary, that are the more adapted for
this model.

Up to now, optimal procedures are based on exponential weights
(cf. \cite{MR2458184}, \cite{dalalyan_tsybakov07}) providing
aggregation procedures with no zero coefficients, even for the worse
elements in the dictionary. An improvement going in the direction of
sparse aggregation has been made using a preselection step in
\cite{LM06}. This preselection step allows to remove all the
estimators in $F$ which performs badly on a learning subsample.

In the present work, we prove that optimal aggregation algorithms with
only two non-zero coefficients exists, see
Section~\ref{sec:ERM_finite}, Theorem~\ref{thm:aggregation}. This
means that the aggregate writes as a convex combination of only two
elements of $F$. Then, we propose an original proof of an already
known result, involving an explicit geometrical setup, of the fact
that selecting a single element of $F$ using empirical risk
minimization is a suboptimal aggregation procedure, see
Theorem~\ref{TheoWeaknessERMRegression}. Finally, we use our
``hyper-sparse'' aggregate on a dictionary ``consisting'' of penalized
empirical risk minimizers (PERM). The aim is to construct, as an
application of the previous analysis, an adaptive estimator over
anisotropic Besov balls, namely an estimator that adapts to the
unknown anisotropic smoothness of the regression function, in the
sense that it achieves the optimal minimax rate without an a priori
knowledge of the anisotropic smoothness parameters. This result was
not, as far as we know, previously proposed in minimax theory. To do
so, we use recent results by \cite{Mendelson08regularizationin} on
regularized learning together with our oracle inequality for the
hyper-spare aggregate.

\subsection{The model}

Let $\Omega$ be a measurable space endowed with a probability measure
$\mu$ and $\nu$ be a probability measure on $\Omega \times \R$ such
that $\mu$ is its marginal on $\Omega$. Assume $(X,Y)$ and $D_n :=
(X_i,Y_i)_{i=1}^n$ to be $n+1$ independent random variables
distributed according to~$\nu$. We work under the following
assumption.
\begin{assumption}
  \label{ass:model}
  We can write
  \begin{equation}
    \label{eq:model}
    Y = f_0(X) + \varepsilon,
  \end{equation}
  where $\varepsilon$ is such that $\E(\varepsilon | X) = 0$ and
  $\E(\varepsilon^2 | X) \leq \sigma_{\varepsilon}^2$ a.s. for some
  constant $\sigma_{\varepsilon} > 0$.
\end{assumption}
We will assume further that either $Y$ is bounded: $\norm{Y}_\infty <
+\infty,$ or that $\varepsilon$ is subgaussian:
$\norm{\varepsilon}_{\psi_2} := \inf \{ c > 0 : \E[ \exp( (\varepsilon
/ c)^2) ] \leq 2 \} < +\infty$, see below. We want to estimate the
regression function $f_0$ using the observations $D_n$. If $f$ is a
function, its error of prediction is given by the risk
\begin{equation*}
  R(f) = \E (f(X) - Y)^2,
\end{equation*}
and if $\hat f$ is a random function depending on the data $D_n$, the
error of prediction is the conditional expectation
\begin{equation*}
  R(\hat{f}) = \E [ (\hat{f}(X) -Y)^2 | D_n].
\end{equation*}
Given a set of functions $F$, a natural way to approximate $f_0$ is to
consider the empirical risk minimizer (ERM), that minimizes the
functional
\begin{equation*}
  R_n(f) := \frac 1n \sum_{i=1}^n (Y_i - f(X_i))^2
\end{equation*}
over $F$. This very basic principle  is at the core of the
procedures proposed in this paper. We will also commonly use the
following notations. If $f^F \in \argmin_{f \in F} R(f)$, we will
consider the excess loss
\begin{equation*}
  \mathcal L_f = \mathcal L_F(f)(X, Y) := (Y - f(X))^2 - (Y -
  f^F(X))^2,
\end{equation*}
and use the notations
\begin{equation*}
  P \mathcal L_f := \E \mathcal L_f(X, Y), \quad P_n \mathcal L_f :=
  \frac 1n \sum_{i=1}^n \mathcal L_f(X_i, Y_i).
\end{equation*}

\section{Hyper-sparse aggregation}
\label{sec:ERM_finite}

\subsection{The aggregation problem}

Assume that we are given a finite set $F = \{ f_1, \ldots, f_M \}$ of
functions (usually called a dictionary), the aggregation problem is to
construct procedures $\tilde{f}$ (usually called an aggregate)
satisfying inequalities of the form
\begin{equation}
  \label{eq:Oracle-inequality-definition}
  R(\tilde{f}) \leq c \min_{f\in F}R(f) + r(F,n),
\end{equation}
where the result holds with high probability or in
expectation. Inequalities of the form
\eqref{eq:Oracle-inequality-definition} are called oracle inequalities
and $r(F,n)$ is called the residue. We want the residue to be as small
as possible. A classical result (cf. \cite{MR2458184}) says that
aggregates with values in $F$ cannot mimic exactly (that is for $c =
1$) the oracle faster than $r(F,n) \sim ((\log M) /
n)^{1/2}$. Nevertheless, it is possible to mimic the oracle up to the
residue $(\log M) / n$ (see \cite{MR2458184} and \cite{LM06}, among
others).

An aggregate typically write as a convex combination of the elements
of $F$, namely
\begin{equation*}
  \hat f := \sum_{j=1}^M \theta_j f_j,
\end{equation*}
where $\theta := (\theta_j(D_n, F))_{j=1}^M$ is a map $\{ 1, \ldots, M
\} \rightarrow \Theta$, where
\begin{equation*}
  \Theta := \Big\{ \lambda \in (\mathbb R^+)^M : \sum_{i=1}^M \lambda_j
  = 1 \Big\}.
\end{equation*}
Popular examples of aggregation algorithms are the aggregate with
cumulated exponential weights (ACEW), see \cite{catbook:01,
  leung_barron06, MR2458184, juditsky_nazin05, audibert-2009-37},
where the weights are given by
\begin{equation*}
  \theta_j^{(\rm ACEW)} := \frac1n \sum_{k=1}^n \frac{\exp(
    -\sum_{i=1}^k (Y_i - f_j(X_i))^2 / T)}{\sum_{l=1}^M \exp(-
    \sum_{i=1}^k (Y_i - f_l(X_i))^2 / T) },
\end{equation*}
where~$T$ is the so-called temperature parameter, and the aggregate
with exponential weights (AEW), see \cite{dalalyan_tsybakov07} among
others, where
\begin{equation*}
  \theta_j^{(\rm AEW)} := \frac{\exp(-\sum_{i=1}^n (Y_i - f_j(X_i))^2
    / T)}{\sum_{l=1}^M \exp(- \sum_{i=1}^n (Y_i - f_l(X_i))^2 / T) }.
\end{equation*}
The ACEW satisfies~\eqref{eq:Oracle-inequality-definition} with $c =
1$ and $r(F, n) \sim (\log M) / n$, see references above, hence it is
optimal in the sense of \cite{tsy:03}. In these aggregates, no
coefficient equals zero, although they can be very small, depending on
the value of $R_n(f_j)$ and~$T$ [this makes in particular the choice
of $T$ of importance]. In this paper, we look for an aggregation
algorithm that shares the same property of optimality, but with as few
non-zero coefficients $\theta_j$ as possible, hence the name
\emph{hyper-sparse aggregate}. We ask for the following question:
\begin{question}
  What is the minimal number of non-zero coefficients $\theta_j$ such
  that an aggregation procedure $\sum_{j=1}^M \theta_j f_j$ is
  optimal?
\end{question}
It turns out that the answer to this question is two. Indeed, if every
coefficient is zero, excepted for one, the aggregate coincides with an
element of $F$, and we know that such a procedure can only achieve the
rate $((\log M) / n)^{1/2}$ (see \cite{MR2458184} and
Theorem~\ref{TheoWeaknessERMRegression} below where, in the particular
case of the ERM, the suboptimality of this kind of procedure can be
understood from a geometrical point of view (this differs from the
statistical point of view from \cite{MR2458184} which involves
``min-max'' type theorem)). In Definition~\ref{def:aggprocedures}, we
construct three procedures, where two of them (see~\eqref{eq:segments}
and~\eqref{eq:starified}), only have two non-zero coefficients
$\theta_j$, and we prove in Theorem~\ref{thm:aggregation} below that
these procedures are optimal. We shall assume one of the following.
\begin{assumption}
  One of the following holds.
  \begin{itemize}
  \item There is a constant $b > 0$ such that:
    \begin{equation}
      \label{ass:bounded}
      \max( \norm{Y}_\infty, \sup_{f \in F} \norm{f}_\infty) \leq b.
    \end{equation}
  \item There is a constant $b > 0$ such that:
    \begin{equation}
      \label{ass:subgaussian}
      \max( \norm{\varepsilon}_{\psi_2}, \norm{\sup_{f \in F} |f(X) -
        f_0(X)|}_{\psi_2} ) \leq b.
    \end{equation}
  \end{itemize}
\end{assumption}
Note that Assumption~\eqref{ass:subgaussian} allows an unbounded
dictionary $F$. The results given below differ a bit depending on the
considered assumption (there is an extra $\log n$ term in the
subgaussian case given by (\ref{ass:subgaussian})). To simplify the
notations, we assume from now that we have $2n$ observations from a
sample $D_{2n} = (X_i,Y_i)_{i=1}^{2n}$. Let us define our aggregation
procedures.

\begin{definition}[Aggregation procedures]
  \label{def:aggprocedures}
  Follow the following steps:
  \begin{description}
  \item [(0. Initialization)] Choose a confidence level $x >
    0$. If~\eqref{ass:bounded} holds, define
    \begin{equation*}
      \phi = \phi_{n, M}(x) = b \sqrt{\displaystyle \frac{\log M +
          x}{n}}.
    \end{equation*}
    If~\eqref{ass:subgaussian} holds, define
    \begin{equation*}
      \phi = \phi_{n, M}(x) = (\sigma_{\varepsilon} + b)
      \sqrt{\displaystyle \frac{(\log M + x) \log n}{n}}.
    \end{equation*}
  \item [(1. Splitting)] Split the sample $D_{2n}$ into $D_{n, 1} =
    (X_i,Y_i)_{i=1}^{n}$ and $D_{n, 2} = (X_i,Y_i)_{i=n+1}^{2n}$.
  \item [(2. Preselection)] Use $D_{n, 1}$ to define a random subset
    of $F:$
    \begin{equation}
      \label{eq:random-set}
      \hat{F}_1 = \Big\{ f \in F : R_{n, 1} (f) \leq R_{n,
        1}( \hat f_{n, 1}) + c \max \big( \phi \norm{\hat f_{n, 1} -
        f}_{n, 1}, \phi^2 \big) \Big\},
    \end{equation}
    where $\norm{f}_{n, 1}^2 = n^{-1} \sum_{i=1}^n f(X_i)^2,$ $R_{n,
      1}(f) = n^{-1} \sum_{i=1}^n (f(X_i) - Y_i)^2,$ $\hat f_{n, 1}
    \in \argmin_{f \in F} R_{n, 1}(f)$.
  \item [(3. Aggregation)] Choose $\FF$ as one of the following sets:
    \begin{align}
      \label{eq:convex_hull}
      \FF &= \conv(\hat F_1) = \text{ the convex hull of }
      \hat F_1 \\
      \label{eq:segments}
      \FF &= \seg(\hat F_1) = \text{ the segments between the
        functions in } \hat F_1 \\
      \label{eq:starified}
      \FF &= \Star(\hat f_{n, 1}, \hat F_1) = \text{ the segments
        between } \hat f_{n, 1} \text{ with the elements of } \hat
      F_1,
    \end{align}
    and return the ERM relative to $D_{n, 2}:$
    \begin{equation*}
      \tilde{f} \in \argmin_{g \in \FF} R_{n, 2}(g),
    \end{equation*}
    where $R_{n, 2}(f) = n^{-1} \sum_{i=n+1}^{2n} (f(X_i) - Y_i)^2.$
  \end{description}
\end{definition}

These algorithms are illustrated in Figures~\ref{fig:aggregates}
and~\ref{fig:aggsimu}. In Figure~\ref{fig:aggregates} we summarize the
aggregation steps in the three cases. In Figure~\ref{fig:aggsimu} we
give a simulated illustration of the preselection step, and we show
the value of the weights of the AEW for a comparison. As mentioned
above, the Step~3 of the algorithm returns, when $\FF$ is given
by~\eqref{eq:segments} or~\eqref{eq:starified}, a function which is a
convex combination of only two functions in $F$, among the ones
remaining after the preselection step. The preselection step was
introduced in \cite{LM06}, with the use of~\eqref{eq:convex_hull} in
the aggregation step.

\begin{figure}[htbp]
  \centering
  \includegraphics[width=12cm]{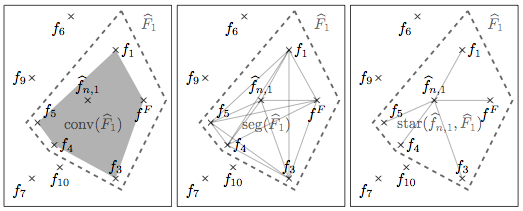}
  \caption{Aggregation algorithms: ERM over $\conv(\hat F_1),$
    $\seg(\hat F_1),$ or $\Star(\hat f_{n, 1}, \hat F_1)$.}
  \label{fig:aggregates}
\end{figure}

Each of the three procedures proposed in
Definition~\ref{def:aggprocedures} are optimal in view of
Theorem~\ref{thm:aggregation} below. From the computational point of
view, procedure~\eqref{eq:starified} is the most appealing: an ERM in
$\Star(\hat f_{n, 1}, \hat F)$ can be computed in a fast and explicit
way, see Algorithm~\ref{alg:star-shaped} below. The next Theorem
proves that each of these aggregation procedures are optimal.

\begin{theorem}
  \label{thm:aggregation}
  Let $x > 0$ be a confidence level, $F$ be a dictionary with
  cardinality~$M$ and $\tilde f$ be one of the aggregation procedure
  given in Definition~\ref{def:aggprocedures}. If 
  \begin{equation*}
    \max( \norm{Y}_\infty, \sup_{f \in F} \norm{f}_\infty) \leq b,
  \end{equation*}
  we have, with $\nu^{2n}$-probability at least $1 - 2 e^{-x}:$
  \begin{equation*}
    R(\tilde{f}) \leq \min_{f \in F} R(f) + c_b \frac{(1 +
      x) \log M}{n},
  \end{equation*}
  where $c_b$ is a constant depending on $b,$ and where we recall that
  $R(\tilde f) = \E[ (Y - \tilde f(X))^2 | (X_i,Y_i)_{i=1}^{2n} ]$. If
  \begin{equation*}
    \max( \norm{\varepsilon}_{\psi_2}, \norm{\sup_{f \in F} |f(X) -
      f_0(X)|}_{\psi_2} ) \leq b,
  \end{equation*}
  we have, with $\nu^{2n}$-probability at least $1 - 4 e^{-x}:$
  \begin{equation*}
    R(\tilde{f}) \leq \min_{f \in F} R(f) + c_{\sigma_{\varepsilon},
      b} \frac{(1 + x)\log M \log n}{n}.
  \end{equation*}
\end{theorem}

\begin{remark}
  Note that the definition of the set $\hat{F}_1$, and thus $\tilde
  f$, depends on the confidence $x$ through the factor $\phi_{n,
    M}(x)$.
\end{remark}

\begin{remark}
  To simplify the proofs, we don't give the explicit values of the
  constants. However, when~\eqref{ass:bounded} holds, one can choose
  $c = 4(1+9b)$ in~\eqref{eq:random-set} and $c =c_1(1+b)$
  when~\eqref{ass:subgaussian} holds (where $c_1$ is the absolute
  constant appearing in Theorem~\ref{thm:adamczak}). Of course, this
  is not likely to be the optimal choice.
\end{remark}

\setlength{\figwidth}{4.5cm} \setlength{\figheight}{5cm}

\begin{figure}[htbp]
  \centering
  \includegraphics[width=\figwidth,height=\figheight]{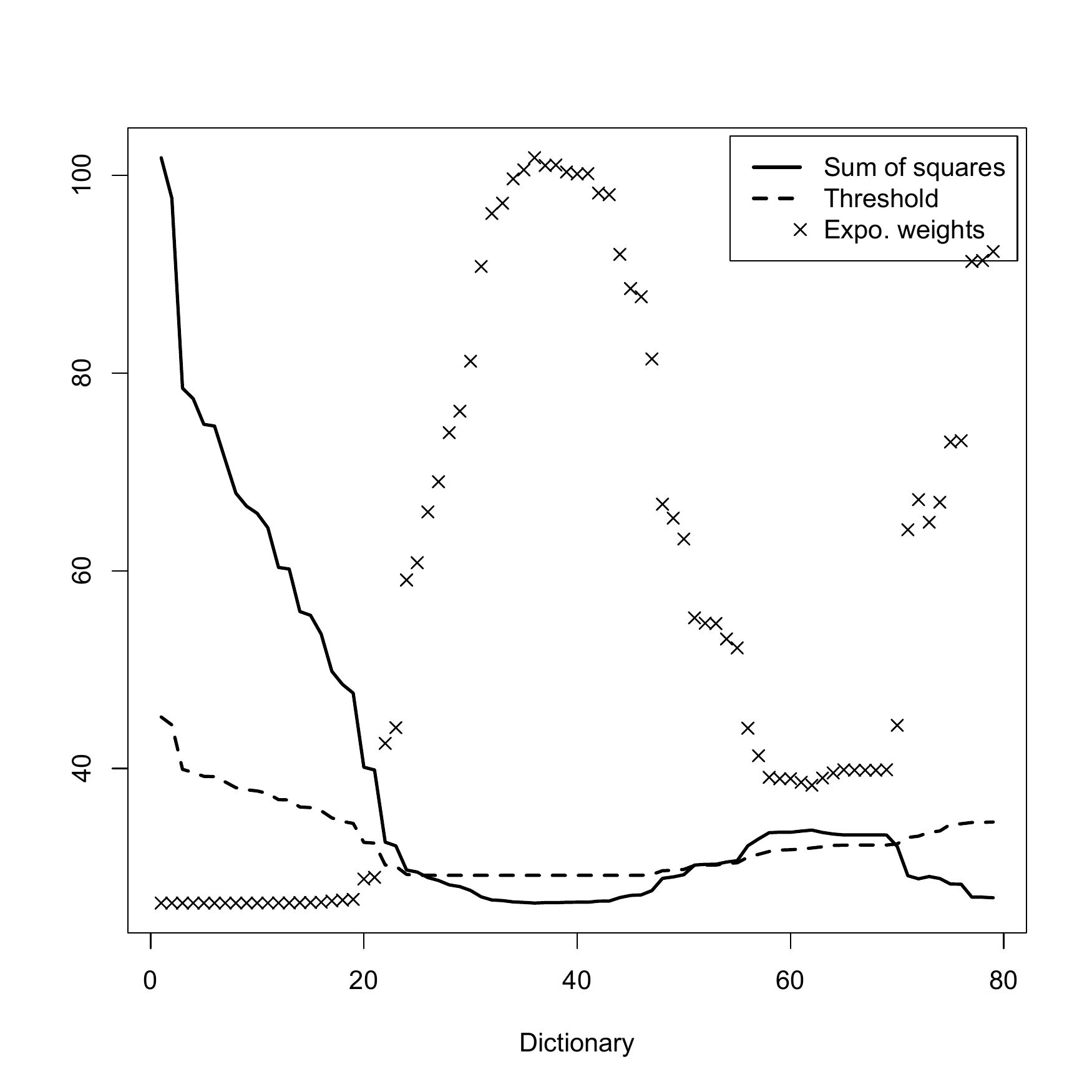}%
  \includegraphics[width=\figwidth,height=\figheight]{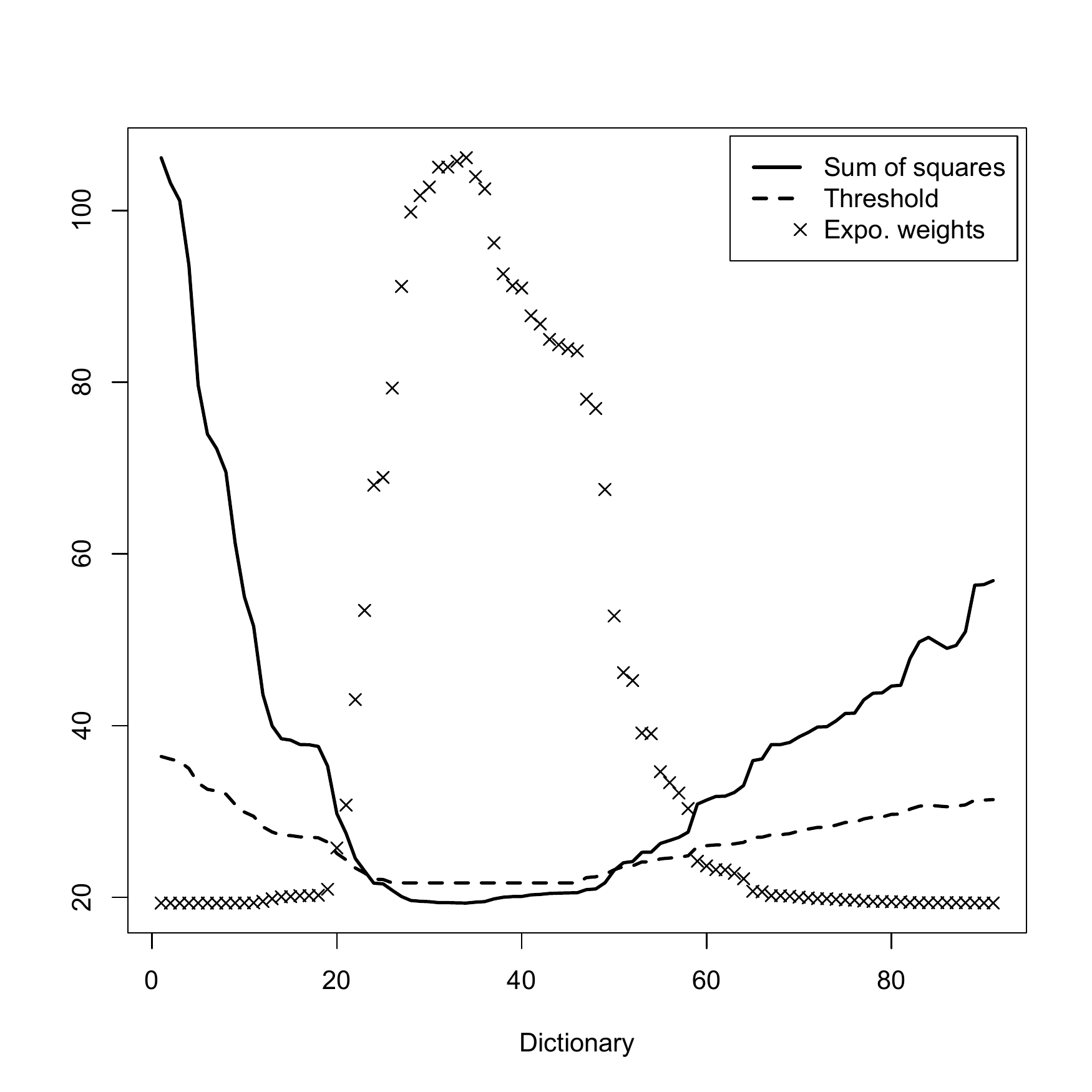}%
  \includegraphics[width=\figwidth,height=\figheight]{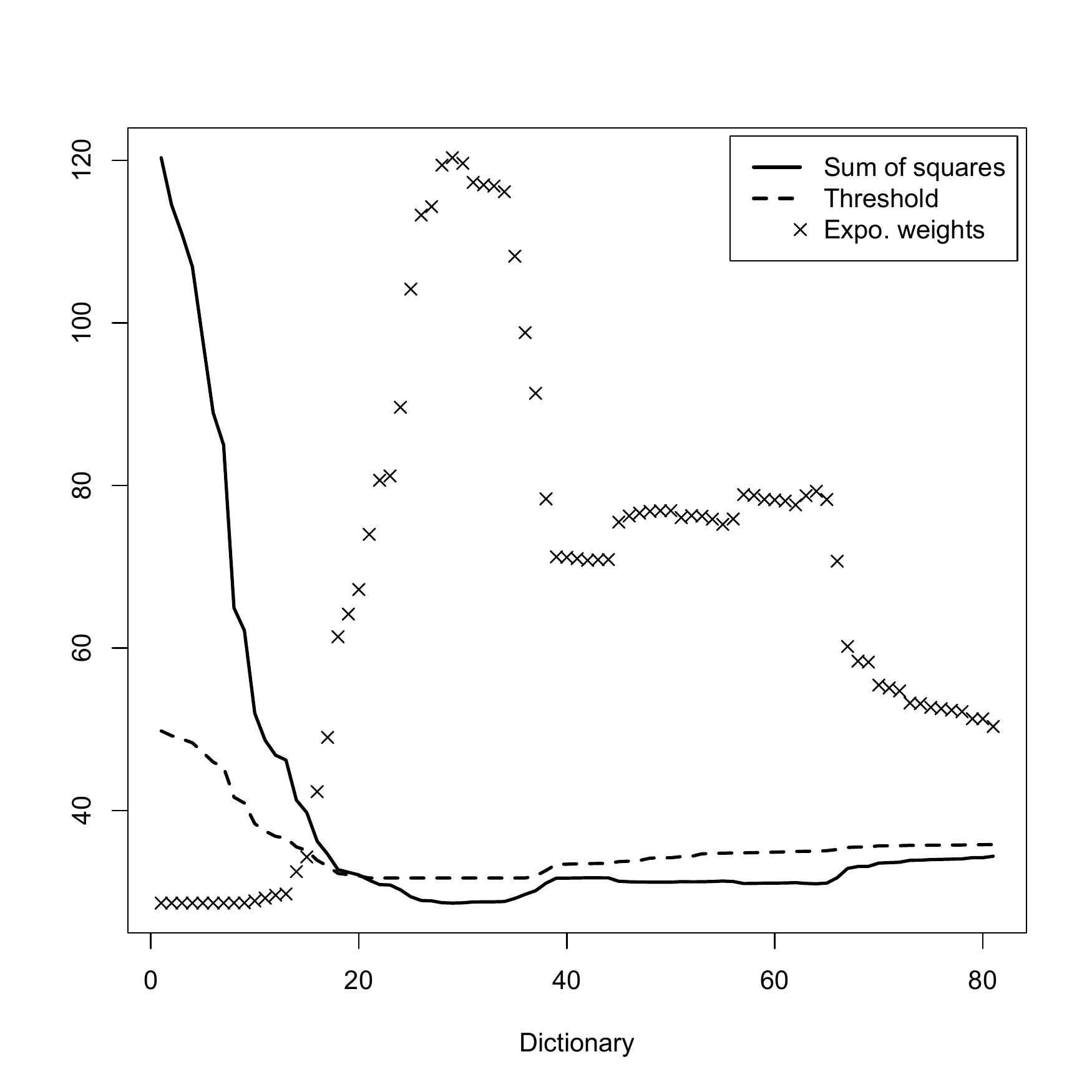}
  \caption{Empirical risk $R_{n, 1}(f)$, value of the threshold $R_{n,
      1}( \hat f_{n, 1}) + 2 \max( \phi \norm{\hat f_{n, 1} - f}_{n,
      1}, \phi^2 )$ and weights of the AEW (that we rescaled for
    illustration purpose) for $f \in F$, where $F$ is a dictionary
    obtained using LARS, see Section~\ref{sec:simu} below. Only the
    elements of $F$ with an empirical risk smaller than the threshold
    are kept from the dictionary, see
    Definition~\eqref{def:aggprocedures}. The first and third examples
    correspond to a case where an aggregate with preselection step
    improves upon AEW, while in the second example, both procedures
    behaves similarly.}
  \label{fig:aggsimu}
\end{figure}

\subsection{The star-shaped aggregate}

In this section we give details for the computation of the star-shaped
aggregate, namely the aggregate $\tilde f$ given by
Definition~\ref{def:aggprocedures} when $\FF$ is~\eqref{eq:starified}.
Indeed, if $\lambda \in [0, 1]$, we have
\begin{equation*}
  R_{n,2}(\lambda f + (1 - \lambda) g) = \lambda R_{n,2}(f) + (1 - \lambda)
  R_{n,2}(g) - \lambda (1 - \lambda) \norm{f - g}_{n,2}^2,
\end{equation*}
so the minimum of $\lambda \mapsto R_{n,2}(\lambda f + (1 - \lambda) g)$
is achieved at
\begin{equation*}
  \lambda_{n,2}(f, g) = 0 \vee \frac 12 \Big( \frac{R_{n,2}(g) - R_{n,2}(f)}{\norm{f -
      g}_{n,2}^2} + 1 \Big) \wedge 1,
\end{equation*}
where $a \vee b = \max(a, b)$, $a \wedge b = \min(a, b)$ and $\min_{\lambda \in [0, 1]} R_{n,2}(\lambda f + (1 - \lambda) g)$ is thus equal to $R_{n,2}(\lambda_{n,2}(f,g) f + (1 - \lambda_{n,2}(f,g)) g)$ given by
\begin{equation*}
   \left\{
    \begin{array}{cc}
R_{n,2}(f) & \mbox{ if } R_{n,2}(f)-R_{n,2}(g)\geq\norm{f-g}_{n,2}^2\\
  \frac{R_{n,2}(f) + R_{n,2}(g)}{2} - \frac{(R_{n,2}(f) -
    R_{n,2}(g))^2}{4 \norm{f - g}_{n,2}^2} - \frac{\norm{f - g}_{n,2}^2}{4}& \mbox{ if }|R_{n,2}(f)-R_{n,2}(g)|\leq\norm{f-g}_{n,2}^2\\
R_{n,2}(g) & \mbox{ otherwise.}
\end{array}\right.
\end{equation*}
This leads to the following algorithm for the computation of $\tilde
f$.

\begin{algorithm}[H]
  \SetLine%
  \KwIn{dictionary $F$, data $(X_i, Y_i)_{i=1}^{2n}$, and a confidence
    level $x > 0$}
  
  \KwOut{star-shaped aggregate $\tilde f$}%
  
  Split $D_{2n}$ into two samples $D_{n, 1}$ and $D_{n, 2}$

  \ForEach{$j \in \{ 1, \ldots, M \}$}{ Compute $R_{n, 1}(f_j)$ and
    $R_{n, 2}(f_j)$, and use this loop to find $\hat f_{n, 1} \in
    \argmin_{f \in F} R_{n, 1}(f)$}%

  \ForEach{$j \in \{ 1, \ldots, M \}$}{ Compute $\norm{f_j - \hat
      f_{n, 1}}_{n, 1}$ and $\norm{f_j - \hat f_{n, 1}}_{n, 2}$}

  Construct the set of preselected elements
    \begin{equation*}
      \hat{F}_1 = \Big\{ f \in F : R_{n, 1} (f) \leq R_{n,
        1}( \hat f_{n, 1}) + c \max \big( \phi \norm{\hat f_{n, 1} -
        f}_{n, 1}, \phi^2 \big) \Big\},
    \end{equation*}    
    where $\phi$ is given in Definition~\ref{def:aggprocedures}.%
  
  \ForEach{$f \in \hat F_1$}{compute
    \begin{equation*}
      R_{n,2}(\lambda_{n,2}(\hat f_{n,1},f) \hat f_{n,1} + (1 - \lambda_{n,2}(\hat f_{n,1},f)) f)
    \end{equation*}
    and keep the element $f_{\hat \jmath} \in \hat F_1$ that minimizes this
    quantity}
  
  \Return
  \begin{equation*}
    \tilde f=\lambda_{n, 2}(\hat f_{n,1},f_{\hat \jmath}) \hat f_{n, 1} + (1 - \lambda_{n, 2}(\hat f_{n,1},f_{\hat \jmath})) f_{\hat j},
  \end{equation*}
  \caption{Computation of the star-shaped aggregate.}
  \label{alg:star-shaped}
\end{algorithm}

\subsection{Suboptimality of Penalized ERM}

In this section, we prove that minimizing the empirical risk
$R_n(\cdot)$ (or a penalized version, called PERM from now on) on
$F(\Lambda)$ is a suboptimal aggregation procedure both in expectation
and deviation. According to \cite{tsy:03}, the optimal rate of
aggregation in the gaussian regression model is $(\log M) /n$. This
means that it is the minimum price one has to pay in order to mimic
the best function among a class of $M$ functions with $n$
observations. This rate is achieved by the aggregate with cumulative
exponential weights, see~\cite{catbook:01}, \cite{MR1762904}
and~\cite{MR2458184}.  In Theorem~\ref{TheoWeaknessERMRegression}
below, we prove that the usual PERM procedure cannot achieve this rate
and thus, that it is suboptimal compared to the aggregation methods
with exponential weights. The lower bounds for aggregation methods
appearing in the literature (see~\cite{tsy:03, MR2458184, LecJMLR:06})
are usually based on minimax theory arguments. In particular, in
\cite{tsy:03}, it is proved that a selector (that is an aggregation
procedure taking its values in the dictionnary itself) cannot mimic
the oracle faster than $\sqrt{(\log M)/n}$. This result implies the
one that we have here, but, it doesn't provide an explicit setup for
which a given selector performs poorly. The result in \cite{MR2458184}
says that whatever the selector is, there exists a probability measure
and a dictionnary for which it cannot mimic the oracle faster than
$\sqrt{(\log M)/n}$. The proof of this result does not tell
explicitely which probabilistic setup is bad for this selector. In the
present result, we are interested in a particular type of selector:
the PERM for some penalty. We can provide an explicit framework
(dictionnary+probabilistic setup) because the argument considered here
is based on some geometric considerations (in the same spirit as the
lower bound obtained in \cite{lbw:96} and \cite{m:08}).  The explicit
example that makes the PERM fail is the following Gaussian regression
model with uniform design:
\begin{assumption}[G]
  Assume that $\varepsilon$ is standard Gaussian and that $X$ is
  univariate and uniformly distributed on $[0, 1]$.
\end{assumption} The dictionary is constructed as follow:

\begin{figure}[htbp]
  \centering
  \includegraphics[width=5cm]{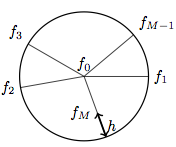}
  \caption{Example of a setup in which ERM performs badly. The set
    $F(\Lambda) = \{f_1, \ldots, f_M \}$ is the dictionary from which
    we want to mimic the best element and $f_0$ is the regression
    function.}
  \label{fig:badsetup}
\end{figure}
For the regression function we take
  \begin{equation}
    \label{FunctionBasisRegression}
    f_0(x) =
    \begin{cases}
      \; 2h &\text{ if } x^{(M)} = 1 \\
      \; h & \text{ if } x^{(M)} = 0,
    \end{cases}
  \end{equation}
  where $x$ has the dyadic decomposition $x=\sum_{k \geq 1}
  x^{(k)}2^{-k}$ where $x^{(k)} \in \{ 0, 1 \}$ and
  \begin{equation*}
    h=\frac{C}{4}\sqrt{\frac{\log M}{n}}.
  \end{equation*}
  We consider the dictionary of functions $F_M = \{f_1, \ldots, f_M\}$
  \begin{equation}
    \label{FunctionBasisRegression}
    f_j(x) = 2x^{(j)}-1, \quad \forall j\in\{1,\ldots,M\},
  \end{equation}
  where again $(x^{(j)} : j \geq 1)$ is the dyadic decomposition of $x
  \in [0,1]$.

\begin{theorem}
  \label{TheoWeaknessERMRegression}
  There exists an absolute constant $c_0>0$ such that the following
  holds.  Let $M \geq 2$ be an integer and assume that \textup{(G)}
  holds. 
  We can find a regression function $f_0$ and a family $F(\Lambda)$ of
  cardinality $M$ such that, if one considers a penalization
  satisfying $|\pen(f)| \leq C \sqrt{(\log M)/n}, \forall f \in
  F(\Lambda)$ with $0\leq C <\sigma (24\sqrt{2}c^*)^{-1}$ \textup($c^*$ is
  an absolute constant from the Sudakov minorization, see
  Theorem~\ref{TheoSudakov} in
  Appendix~\ref{sec:appendix_proba}\textup), the PERM procedure
  defined by
  \begin{equation*}
    \tilde{f}_n \in \argmin_{f \in F(\Lambda)}( R_n(f) + \pen(f))
  \end{equation*}
  satisfies, with probability greater than $c_0$,
  \begin{equation*}
    \| \tilde{f}_n - f_0 \|^2 \geq \min_{f \in
      F(\Lambda)} \| f - f_0 \|^2 + C_3 \sqrt{\frac{\log
        M}{n}}
  \end{equation*}
  for any integer $n \geq 1$ and $M\geq M_0(\sigma)$ such that $n^{-1}
  \log[(M-1)(M-2)] \leq 1/4$ where $C_3$ is an absolute constant.
\end{theorem}
This result tells that, in some particular cases, the PERM cannot
mimic the best element in a class of cardinality $M$ faster than
$((\log M)/n)^{1/2}$. This rate is very far from the optimal one
$(\log M)/n$. Of course, one can say that the PERM fails to achieve
the optimal rate only in the very particular framework that we have
constructed here. Nevertheless, this approach can be generalized (we
refer the reader to \cite{LM2} for instance). Finally, remark that
classical penalty functions are of the order [Complexity of the class]
divided by $n$, which is in our aggregation setup of the order of
$(\log M)/n$. Thus, the restriction that we have on the penalty
function covers the classical cases that one can meet in the
litterature on penalization methods.

Let $F(\Lambda)$ be the set that we consider in the proof of
Theorem~\ref{TheoWeaknessERMRegression} (see
Section~\ref{sec:proof_main_results} below), and take $\pen(f) = 0$.
Using Monte-Carlo (we do $5000$ loops), we compute the excess risk $E
\| \tilde{f}_n - f_0 \|^2 - \min_{f \in F(\Lambda)} \| f - f_0 \|^2$
of the ERM. In Figure~\ref{fig:subERM} below, we compare the excess
risk and the bound $((\log M) / n)^{1/2}$ for several values of $M$
and $n$. It turns out that, for this set $F(\Lambda)$, the lower bound
$((\log M) / n)^{1/2}$ is indeed accurate for the excess
risk. Actually, by using the classical symmetrization argument and the
Dudley's entropy integral (or Pisier's inequality), it is easy to
obtain an upper bound for the excess risk of the ERM of the order of
$((\log M) / n)^{1/2}$ for any class $F(\Lambda)$ of cardinality $M$.

\begin{figure}[htbp]
  \centering
  \includegraphics[width=4.3cm]{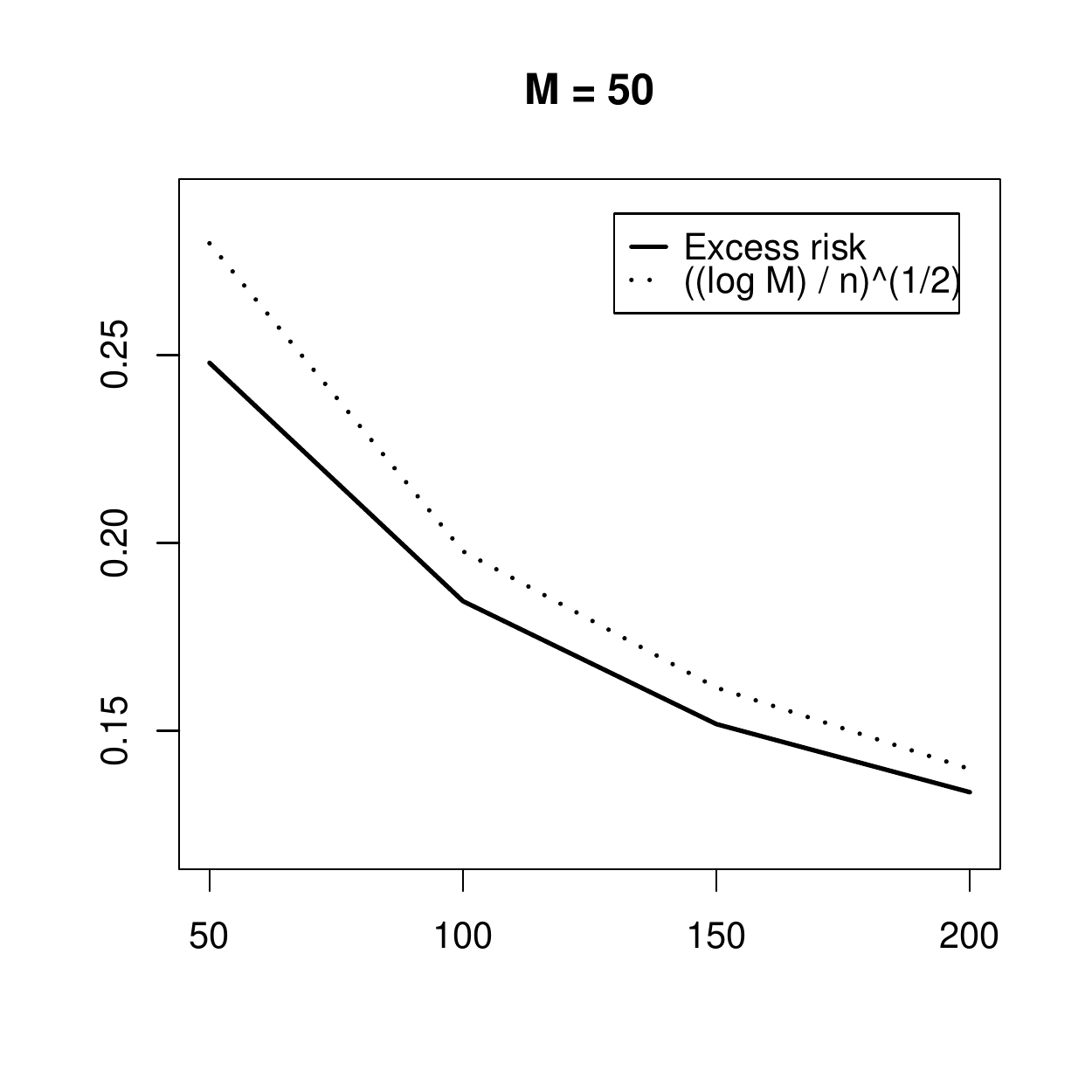}%
  \includegraphics[width=4.3cm]{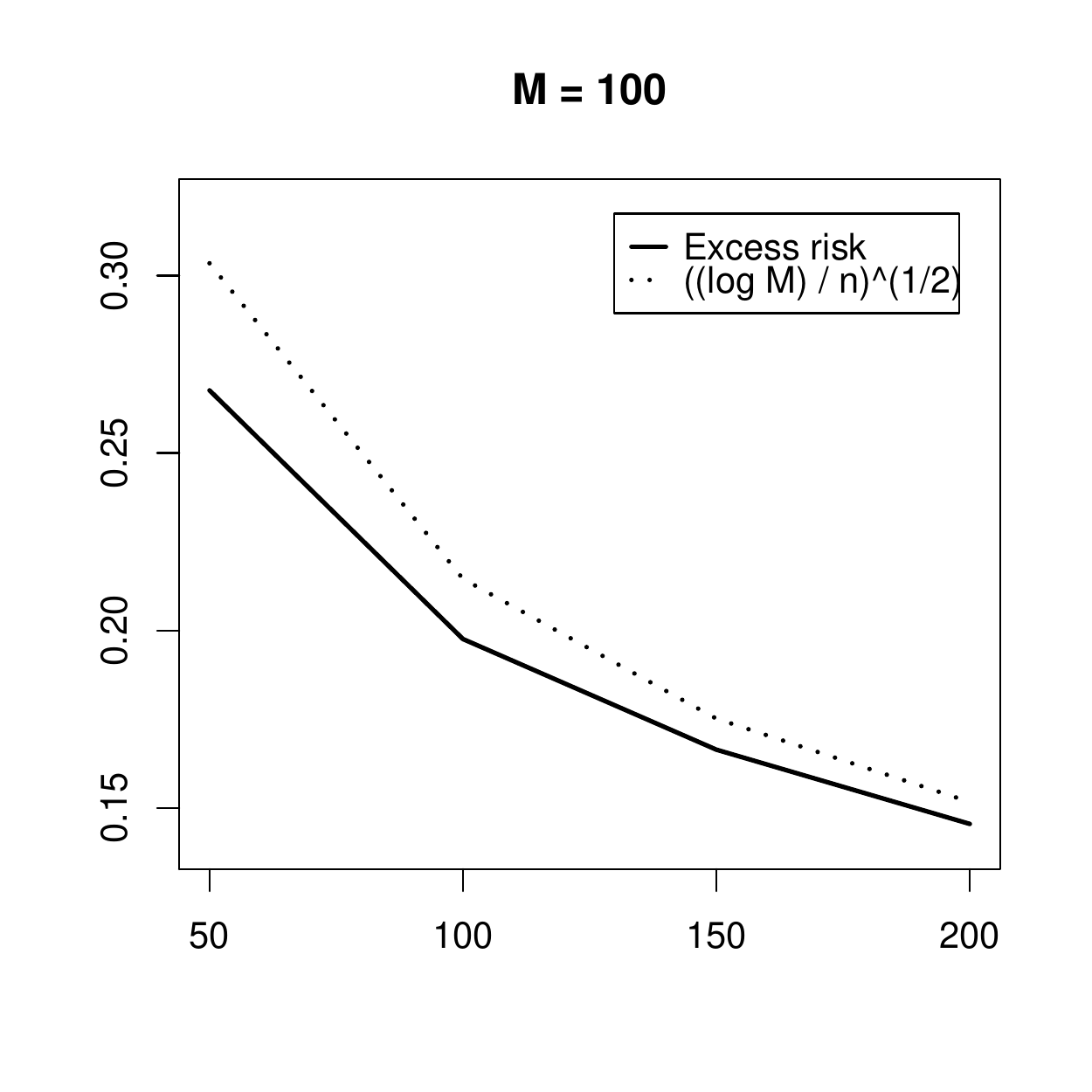}%
  \includegraphics[width=4.3cm]{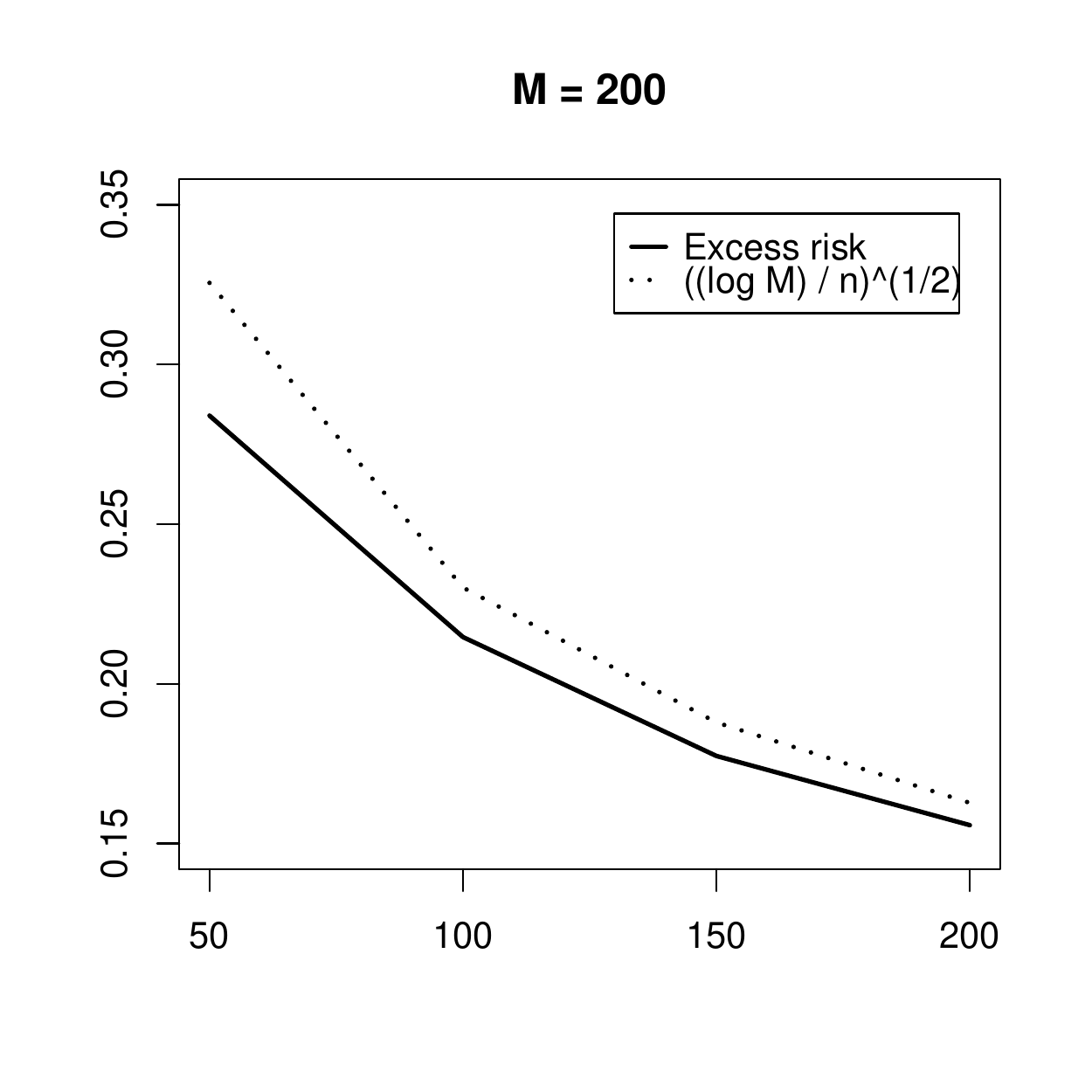}%
  \caption{The excess risk of the ERM compared to $((\log M) /
    n)^{1/2}$ for several values of $M$ and $n$
    \textup($x$-axis\textup)}
  \label{fig:subERM}
\end{figure}

As an application of the aggregation algorithm
\ref{def:aggprocedures}, we consider the problem of adaptation to the
regularization parameter of a penalized empirical risk minimization
procedure, denoted for short PERM in what follows.

\section{An example of dictionary: Penalized ERM}
\label{sec:pena_least_squares}

\subsection{Definition and tools}

Let us fix a function space $\mathcal F$, endowed with a seminorm
$|\cdot|_{\mathcal F}$. The set $\mathcal F$ is a space of functions,
such as a Sobolev, Besov or Reproducing Kernel Hilbert Space (RHKS),
the latter being a common example in regularized learning, see
\cite{cucker_smale02}. A simple example (in the one dimensional case)
is the Sobolev space $W_2^s$ of functions such that $|f|_{\mathcal
  F}^2 = \int f^{(s)}(t)^2 dt < +\infty$, which corresponds to the
so-called \emph{smoothing splines} estimator, see \cite{wahba90}]. A
PERM (which stands for penalized empirical risk minimization)
minimizes the functional
\begin{equation}
  \label{eq:PERM}
  \frac 1n \sum_{i=1}^n (Y_i - f(X_i))^2 + \pen(f)
\end{equation}
over $\mathcal F$, where $\pen(f)$ is a quantity measuring the
smoothness (or ``roughness'') of $f \in \mathcal F$. Typically, the
penalization term writes $\pen(f) = h^2 |f|_{\mathcal F}^2$ (see
\cite{van_de_geer00} and \cite{kohler02} among others), where $h > 0$
is a regularization parameter.

In~\cite{Mendelson08regularizationin}, sharp error bounds for the PERM
are established, in the general context of a so-called \emph{ordered
  and parametrized hierachy} $\{ \mathcal F_r : r > 0\}$. An example
of such an ordered and parametrized hierachy is
\begin{equation*}
  \mathcal F_r = r \mathcal F_1, \text{ where } \mathcal F_1 = \{ f
  \in \mathcal F : |f|_{\mathcal F} \leq 1 \}.
\end{equation*}
In the latter paper, a very sharp analysis is conducted when $\mathcal
F$ is a RKHS, allowing for penalizations less than quadratic in the
RKHS norm. In this section, we use the tools proposed in
\cite{Mendelson08regularizationin} to derive an error bound for the
PERM using the standard penalty $\pen(f) = h^2 |f|_{\mathcal F}^2$,
but when $\mathcal F$ is a Besov space. In nonparametric estimation
literature, Besov spaces are of particular interest since they include
functions with \emph{inhomogeneous smoothness}, for instance functions
with rapid oscillations or bumps. Moreover, since the design random
variable $X$ is eventually multivariate, the question of anisotropic
smoothness naturally arises.  Anisotropy means that the smoothness of
the regression function $f_0$ differs in each direction.  As far as we
know, adaptive estimation of a multivariate curve with anisotropic
smoothness was previously considered only in Gaussian white noise or
density models, see~\cite{kerk_lepski_picard01},
\cite{hoffmann_lepski02}, \cite{kerk_lepski_picard07},
\cite{neumann00}. There is no result concerning the adaptive
estimation of the regression with anisotropic smoothness on a general
random design $X$. In order to simplify the definition of the
anisotropic Besov space, we shall assume from now that $\Omega =
\mathbb R^d$. Let us consider the following compactness assumption on
the unit ball $\mathcal F_1$. It uses metric entropy, which is a
standard measure of the compactness in learning theory, see
\cite{cucker_smale02} for instance. Recall that $\norm{f}_{\infty} =
\sup_{x \in \mathbb R^d} |f(x)|$, and denote by $C(\mathbb R^d)$ the
set of continuous functions on $\mathbb R^d$, endowed with the
$L^\infty$-norm. If $\mathcal F_1 \subset C(\mathbb R^d)$, we
introduce $H_\infty( \mathcal F_1, \delta) = \log N_\infty( \mathcal
F_1, \delta )$, where $N_\infty( \mathcal F_1, \delta )$ is the
minimal number of $L^\infty$-balls with radius $\delta$ needed to
cover $\mathcal F_1$.
\begin{assumption}[$C_\beta$]
  \label{ass:entropy}
  Assume that $\mathcal F$ embeds continuously in $C(\mathbb R^d),$
  and that there is a number $\beta \in (0, 2)$ such that for any
  $\delta > 0$, the unit ball of $\mathcal F$ satisfies:
  \begin{equation}
    H_\infty ( \delta, \mathcal F_1 ) \leq c \delta^{-\beta},
  \end{equation}
  where $c > 0$ is independent of $\delta$.
\end{assumption}
This assumption entails $H_\infty ( \delta, \mathcal F_r ) \leq c (r /
\delta )^{\beta}$ for any $r > 0$. Moreover, the continuous embedding
gives that $\norm{f}_\infty \leq c |f|_{\mathcal F}$ for any $f \in
\mathcal F$. Assumption~\ref{ass:entropy} is satisfied by barely all
the smoothness spaces considered in nonparametric literature (at least
when the smoothness of the space is large enough compared to the
dimension, see below). Let us give an example. Let $B_{p, q}^{\bs s}$
be the anisotropic Besov space with smoothness $\bs s = (s_1, \ldots,
s_d)$. This space is precisely defined in
Appendix~\ref{sec:appendix_approximation}. Each $s_i$ corresponds to
the smoothness in the $i$-th coordinate. The computation of the
entropy of $\mathcal F_1$ when $\mathcal F = B_{p, q}^{\bs s}$ is done
in Theorem~5.30 from~\cite{triebel06}. Namely, if $\bs {\bar s}$ is
the harmonic mean of $\bs s$, given by
\begin{equation}
  \label{eq:harmonic_mean}
  \frac{1}{\bs {\bar s}} := \frac{1}{d} \sum_{i=1}^d
  \frac{1}{s_i},
\end{equation}
then the unit ball $\mathcal F_1$ of $\mathcal F = B_{p, q}^{\bs s}$
satisfies Assumption~\ref{ass:entropy} with $\beta = d / \bs {\bar
  s}$, given that $\bs {\bar s} > d / p$, which is the usual condition
to have the embedding in $C(\mathbb R^d)$.


\subsection{Local complexity using entropy}

Now, we have in mind to use Theorem~2.5 from
\cite{Mendelson08regularizationin}, in order to derive a risk bound
for the PERM. For this, we need a control on the local complexity of
$\mathcal F_r$, for any $r > 0$. The complexity is measured in this
paper by the expectation $\E\norm{P - P_n}_{V_{r, \lambda}}$, where
for $r, \lambda > 0$, $V_{r, \lambda}$ is the class of excess losses
\begin{equation*}
  V_{r, \lambda} := \{ \alpha \mathcal L_{r, f} : 0 \leq \alpha \leq
  1, f \in \mathcal F_r, \E (\alpha\mathcal L_{r, f}) \leq \lambda \},
\end{equation*}
where
\begin{equation*}
  \mathcal L_{r, f} := (Y - f(X))^2 - (Y - f_r^*(X))^2
\end{equation*}
and $f_r^* \in \argmin_{f \in \mathcal F_r} \E(Y - f(X))^2$. The next
Lemma (a proof is given in Appendix A.3) gives a bound on this measure
of the complexity under Assumption~\ref{ass:entropy}.
\begin{lemma}
  \label{lem:complexity}
  Assume that $\norm{Y}_\infty < +\infty$ and grant
  Assumption~\ref{ass:entropy}. One has, for any $r, \lambda > 0$:
  \begin{equation*}
    \E\norm{P - P_n}_{V_{r, \lambda}} \leq c \max\Big[ r^2 n^{-1
      / (1 + \beta /2)}, \frac{r^{1 + \beta / 2} \lambda^{(1 -
        \beta / 2) / 2}}{\sqrt n} \Big],
  \end{equation*}
  where $c = c_{\beta, \norm{Y}_\infty}$.
\end{lemma}
This Lemma, although probably not optimal, is sufficient to provide a
satisfactory risk bound for the PERM with a penalization of the form
$\pen(f) = h^2 |f|_{\mathcal F}^2$. It is close in spirit to a bound
proposed in \cite{loustau09} (see Theorem~1) for the problem of
classification framework using a Besov penalization, with an extra
assumption on the inputs $X_i$, since the proof involves a
decomposition on a wavelet basis. Here we use only the entropy
condition, together with some basic tools from empirical process
theory, see the proof in Appendix~\ref{sec:lemmas}.

\subsection{A risk bound for the PERM using entropy}

Now, we can derive a risk bound for the PERM using
Lemma~\ref{lem:complexity} and the results from
\cite{Mendelson08regularizationin}. First, note that 
\begin{equation*}
  \lambda / 8 \geq c \max\Big[ r^2 n^{-1
    / (1 + \beta /2)}, \frac{r^{1 + \beta / 2} \lambda^{(1 -
      \beta / 2) / 2}}{\sqrt n} \Big]
\end{equation*}
if and only if $\lambda \geq c r^2 n^{-1 / (1 + \beta /2)}$. So, any
$\lambda \geq c r^2 n^{-1 / (1 + \beta /2)}$ satisfies, using
Lemma~\ref{lem:complexity}, that $\lambda / 8 \geq \E\norm{P -
  P_n}_{V_{r, \lambda}}$, and consequently, using the ``isomorphic
coordinate projection'' [see Theorem~2.2 in
\cite{Mendelson08regularizationin}], we have that for any $f \in
\mathcal F_r$, the following holds w.p. larger than $1 - 2e^{-x}$:
\begin{equation*}
  \frac 12 P_n \mathcal L_{r, f} - \rho_n(r, x) \leq P \mathcal L_{r, f}
  \leq 2 P_n \mathcal L_{r, f} + \rho_n(r, x),
\end{equation*}
where
\begin{equation}
  \label{eq:rho_n}
  \rho_n(r, x) := c \Big( r^2 n^{-1 / (1 + \beta / 2)} + \frac{(1 + r^2)
    x}{n} \Big). 
\end{equation}
This explains the shape of the usual quadratic penalization $\pen(f) =
h^2 |f|_{\mathcal F}^2$, where $h = c n^{-1 / (2 + \beta)}$ (up to the
other term, which is of smaller order $1/n$), and this entails the
following.
\begin{theorem}
  \label{thm:PERM}
  Assume that $\norm{Y}_\infty \leq b$, and grant
  Assumption~\ref{ass:entropy}. Let $\rho_n(r, x)$ be given
  by~\eqref{eq:rho_n} and define for $r, y > 0$:
  \begin{equation*}
    \theta(r, y) = y + \log(\pi^2 / 6) + 2 \log(1 + cn + \log
    r),
  \end{equation*}
  where $c = c_{\beta, b}$. Then, for any $x > 0$, with probability at
  least $1 - 2\exp(-x)$, any $\bar f \in \mathcal F$ that minimizes the
  functional
  \begin{equation*}
    P_n \ell_f + c_1 \rho_n( 2 |f|_{\mathcal F}, \theta(|f|_{\mathcal
      F}, x) )
  \end{equation*}
  over $\mathcal F$ also satisfies
  \begin{equation*}
    P \ell_{\bar f} \leq \inf_{f \in \mathcal
      F} \Big( P \ell_f + c_2 \rho_n(2 |f|_{\mathcal F},
    \theta(|f|_{\mathcal F}, x) ) \Big).
  \end{equation*}
\end{theorem}

\begin{proof}
  The conditions of Theorem~2.5 in \cite{Mendelson08regularizationin}
  are satisfied with $\rho_n(r, x)$ given by~\eqref{eq:rho_n}. The
  statement of the Theorem easily follows from it, using the same
  arguments as in the proof of Theorem~3.7 herein.
\end{proof}

Let us rewrite the result of Theorem~\ref{thm:PERM}. For any $x > 0$,
if $\bar f$ is the PERM at level $x$, one has, with
$\nu^n$-probability larger than $1 - 2e^{-x}$:
\begin{align*}
  P \ell_{\bar f} \leq \inf_{r > 0} \Big\{ P \ell_{f_r} + c_1 r^2
  n^{-\frac{1}{1 + \beta/2}} + \frac{c_2(1 + r^2)}{n} \big(x +
  \log(\frac{\pi^2}{6}) + \log( 1 + c_3 n + \log r ) \big) \Big\},
\end{align*}
where we recall that $P \ell_{\bar f} = \E[ (Y - \bar f(X))^2 | X_1,
\ldots, X_n ]$ and $f_r \in \argmin_{f \in \mathcal F_r} R(f)$. This
inequality proves that $\bar f$ adapts to the radius $|f|_{\mathcal
  F}$ of $f$ in $\mathcal F$. The leading term in the right hand side
of this inequality is $r^2 n^{-2 / (2 + \beta)}$. If $\mathcal F =
B_{p, \infty}^{\bs s}$, it becomes $r^2 n^{-2 \bar {\bs s} / (2 \bar
  {\bs s} + d)}$, which is the minimax optimal rate of convergence
over anisotropic Besov space, see \cite{kerk_lepski_picard07} for
instance.

\subsection{Adaptive estimation over anisotropic Besov space}

What we have in mind now is the application of
Theorems~\ref{thm:aggregation} and~\ref{thm:PERM} to the problem of
adaptive estimation over a collection of anisotropic Besov
space. Consider two vectors $\bs s^{\min}$ and $\bs s^{\max}$ in
$\mathbb R_+^d$ with positive coordinates and harmonic means $\bar
{\bs s}^{\min}$ and $\bar {\bs s}^{\max}$ respectively, satisfying
$\bs s^{\min} \leq {\bs s}^{\max}$ ($s_i^{\min} \leq s_i^{\max}$ for
any $i \in \{ 1, \ldots, d \}$) and $\bar {\bs s}^{\min} > d / \min(p,
2)$. Consider the collection of anisotropic Besov space
\begin{equation}
  \label{eq:besov-collection}
  ( B_{p, \infty}^{\bs s} : \bs s \in \bs S ), \text{ where }
  \bs S := \prod_{i=1}^d [s_i^{\min}, s_i^{\max}].
\end{equation}
The strategy is to aggregate a dictionary of PERM, corresponding to a
discretization of $\bs S$, in order to adapt to the anisotropic
smoothness of $f_0$. The steps are the following. We shall assume to
simplify that we have $2n$ observations.
\begin{definition}[Adaptive estimator]
  \label{def:adaptive-est}
  \
  \begin{enumerate}
  \item Split (at random) the whole sample $(X_i, Y_i)_{i=1}^{2n}$
    into a training sample $(X_i, Y_i)_{i=1}^{n}$ and a learning
    sample $(X_i, Y_i)_{i=n+1}^{2n}$. Fix a confidence level $x > 0$.
  \item Compute the uniform discretization of $\bs S$ with step $(\log
    n)^{-1}$:
    \begin{equation}
      \label{eq:discr_smoothness_cube}
      \bs S_n := \prod_{i=1}^d  \big\{ s_i^{\min}
      + k (\log n)^{-1} :1\leq k \leq [ (s_i^{\max} - s_i^{\min}) \log n ]
      \big\}.
    \end{equation}
    Then, for each $\bs s \in \bs S_n$, take $\bar f_s$ as a minimizer
    of the functional
    \begin{equation*}
      \frac{1}{n} \sum_{i=1}^n (Y_i - f(X_i))^2 + \pen_{\bs s}(f, x),
    \end{equation*}
    where
    \begin{align*}
      \pen_{\bs s}(f, x) &= c_1 n^{ -2 \bar {\bs s} / (2 \bar {\bs
          s} + d)} |f|_{B_{p, \infty}^{\bs s}}^2  \\
      &+ \frac{c_2(1 + |f|_{B_{p, \infty}^{\bs s}}^2)}{n} \big(x +
      \log(\frac{\pi^2}{6}) + \log( 1 + c_3 n + \log |f|_{B_{p,
          \infty}^{\bs s}} ) \big).
    \end{align*}
    If $b$ is such that $\norm{Y}_\infty \leq b$, consider the
    dictionary of truncated PERM
    \begin{equation*}
      F^{\rm PERM} = \{ -b \vee \bar f_s \wedge b : s \in \bs S_n \}.
    \end{equation*}
  \item Using the learning sample $(X_i, Y_i)_{i=n+1}^{2n}$, compute
    one of the aggregates $\tilde f$ given in
    Definition~\ref{def:aggprocedures} using the dictionary $F^{\rm
      PERM}$.
  \end{enumerate}
\end{definition}

The next Theorem, which is an immediate consequence of
Theorems~\ref{thm:aggregation} and~\ref{thm:PERM}, proves that the
aggregate $\tilde f$ is minimax adaptative over the collection of
anisotropic Besov spaces~\eqref{eq:besov-collection}.

\begin{theorem}
  \label{thm:adaptive}
  Let $\tilde f$ be the aggregated estimator given in
  Definition~\ref{thm:adaptive}. Assume that $\max( \norm{Y}_\infty,
  \norm{f_0})_{\infty} \leq b$ and that $f_0 \in B_{p, \infty}^{\bs
    s_0}$ for some $\bs s_0 \in \bs S$, where $(B_{p, \infty}^{\bs s}
  : \bs s \in \bs S)$ is the collection given
  by~\eqref{eq:besov-collection} that satisfies $\bar {\bs s}^{\min} >
  d / p$. Then, with $\nu^{2n}$-probability larger than $1 - 4
  e^{-x}$, we have:
  \begin{align*}
    \norm{&\tilde f - f_0}_{L^2(\mu)}^2 \\
    &\leq c_1 r_0^2 n^{-\frac{2 \bar{\bs s}_0}{2 \bar{\bs s}_0 + d}} +
    c_2 \frac{1 + r_0^2 + \log \log n}{n} \big(x + \log(\pi^2 / 6) + c
    \log( 1 + c_3 n + \log r_0 ) \big) \Big\},
  \end{align*}
  where $r_0 = |f_0|_{B_{p, \infty}^{\bs s}}$ and
  \begin{equation*}
    \frac{1}{\bar{\bs s}_0 } = \frac 1d \sum_{i=1}^d \frac{1}{s_{0,
        i}}.
  \end{equation*}
\end{theorem}

The dominating term in the right hand side is of order $n^{-\frac{2
    \bar{\bs s}_0}{2 \bar{\bs s}_0 + d}}$, which is the minimax
optimal rate of convergence over anisotropic Besov space (a minimax
lower bound over $B_{p, q}^{\bs s}$ can be easily obtained using
standard arguments, such as the ones from~\cite{tsybakov03}, together
with Bernstein estimates over $B_{p, \infty}^{\bs s}$ (that can be
found in~\cite{triebel06} for instance). Note that there is no regular
or sparse zone here, since the error of estimation is measured with
$L^2(\mu)$ norm. The result obtained here is stronger than the ones
usually obtained in minimax theory, where one only gives an upper
bound for $\E \norm{\tilde f - f_0}_{L^2(\mu)}^2$, while here is given
a concentration inequality for $\norm{\tilde f - f_0}_{L^2(\mu)}^2$.

\section{Simulation study}
\label{sec:simu}

In this section, we propose a simulation study for the problem of
selection of the smoothing parameter of the LASSO, see
\cite{MR1379242, MR2060166}. We simulate i.i.d. data
\begin{equation*}
  Y_i = \beta_0^\top X_i + \varepsilon_i,
\end{equation*}
where $\beta_0$ is a vector of size $p = 91$ given by
\begin{equation*}
  \beta_0 = (3, 1.5, 0^{30}, 2, -6, 4, 0^{25}, -4, 0^{15}, 2.5,
  3, 0^{10}, 3, 1, -2)
\end{equation*}
where $0^n$ is the vector in $\mathbb R^n$ with each coordinate set to
zero. The noise $\varepsilon_i$ is centered Gaussian with variance
$\sigma^2$. The vector $X = (X^1, \ldots, X^d)$ is a centered Gaussian
vector such that the correlation between $X^i$ and $X^j$ is
$2^{-|i-j|}$ (following the examples from \cite{MR1379242}). Using the
\texttt{lars} routine from
\texttt{R}\footnote{\texttt{www.r-project.org}}, we construct a
dictionary $F$ made of the entire sequence of LASSO type estimators for various regularization parameters coming out of the LARS algorithm.  Then we compare the prediction error $| \bs X(\hat
\beta - \beta_0) |_2$ and the estimation error $| \hat \beta - \beta_0 |_2$ where $\hat
\beta$ is:
\begin{itemize}
\item $\hat \beta^{(C_p)} = $ the LASSO with regularization parameter
  selected using Mallows-$C_p$ selection rule, see \cite{MR2060166}
\item $\hat \beta^{(\rm AEW)} = $ The aggregate with exponential
  weights computed on $F$ with temperature parameter $4 \sigma^2$, see
  for instance~\cite{dalalyan_tsybakov07}
\item $\hat \beta^{(\rm star)} = $ the star-shaped aggregate, see
  Algorithm~\ref{alg:star-shaped}, with constant $c = 2$.
\end{itemize}
We compute the errors $| \bs X(\hat \beta - \beta_0) |_2$ and $|\hat
\beta - \beta_0|_2$ using $100$ simulations for several values of $n$
and $\sigma^2$. The splits taken are chosen at random with size $n/2$
for training and $n/2$ for learning for both the AEW and star-shaped
aggregate (we don't split the learning sample). For both aggregates we
do some jackknife: instead of using a single aggregate, we compute a
mean of~10 aggregates obtained with several splits chosen at
random. This makes the final aggregates less dependent on the
split. In order to make the oracle and the Mallows-$C_p$ errors
comparable to the error of the aggregates (that need to split the
data, while Mallows-$C_p$ doesn't), we compute the weights of
aggregation using splitting, then we compute the aggregate using a
dictionary $F$ computed using the whole sample.

The conclusion is that, for this example, the star-shaped does a
better job than both the AEW and the $C_p$ in most cases. When the
noise level is not too high ($\sigma=2$, which corresponds to a RSNR
of~5), see the errors given in Figure~\ref{fig:sigma2}, the
star-shaped is always the best. When the noise level is high
($\sigma=5$, RSNR=2) and $n$ is small, see Figure~\ref{fig:sigma5},
the story is different: the AEW is better than the star-shaped. In
such an extreme situation the AEW takes advantage of the averaging
(recall that no coefficient is zero in the AEW). However, when $n$
becomes larger than $p$, the star-shaped improves again upon AEW.

\setlength{\figwidth}{4.5cm}
\setlength{\figheight}{6cm}

\begin{figure}[htbp]
  \centering
  \includegraphics[width=\figwidth,height=\figheight]{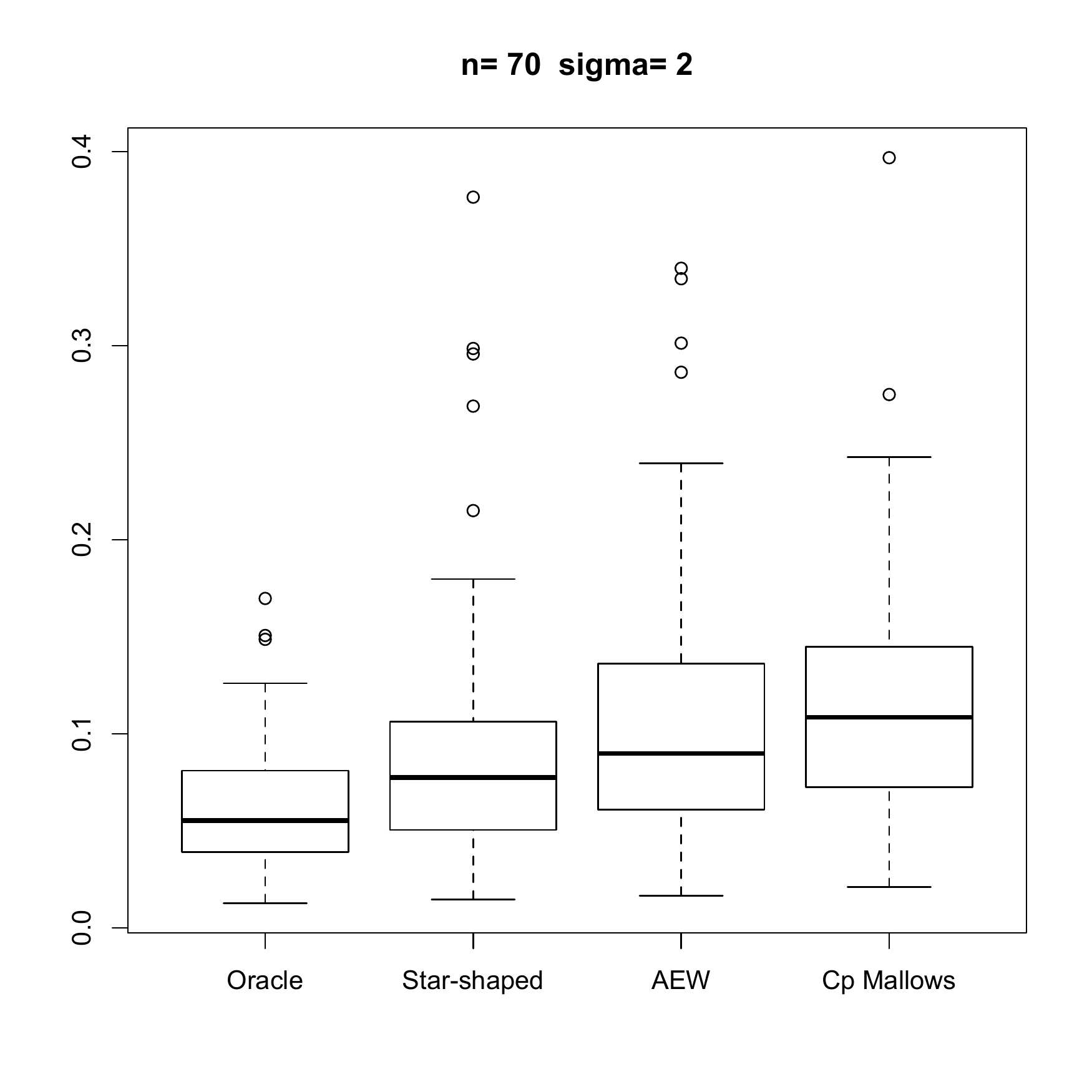}%
  \includegraphics[width=\figwidth,height=\figheight]{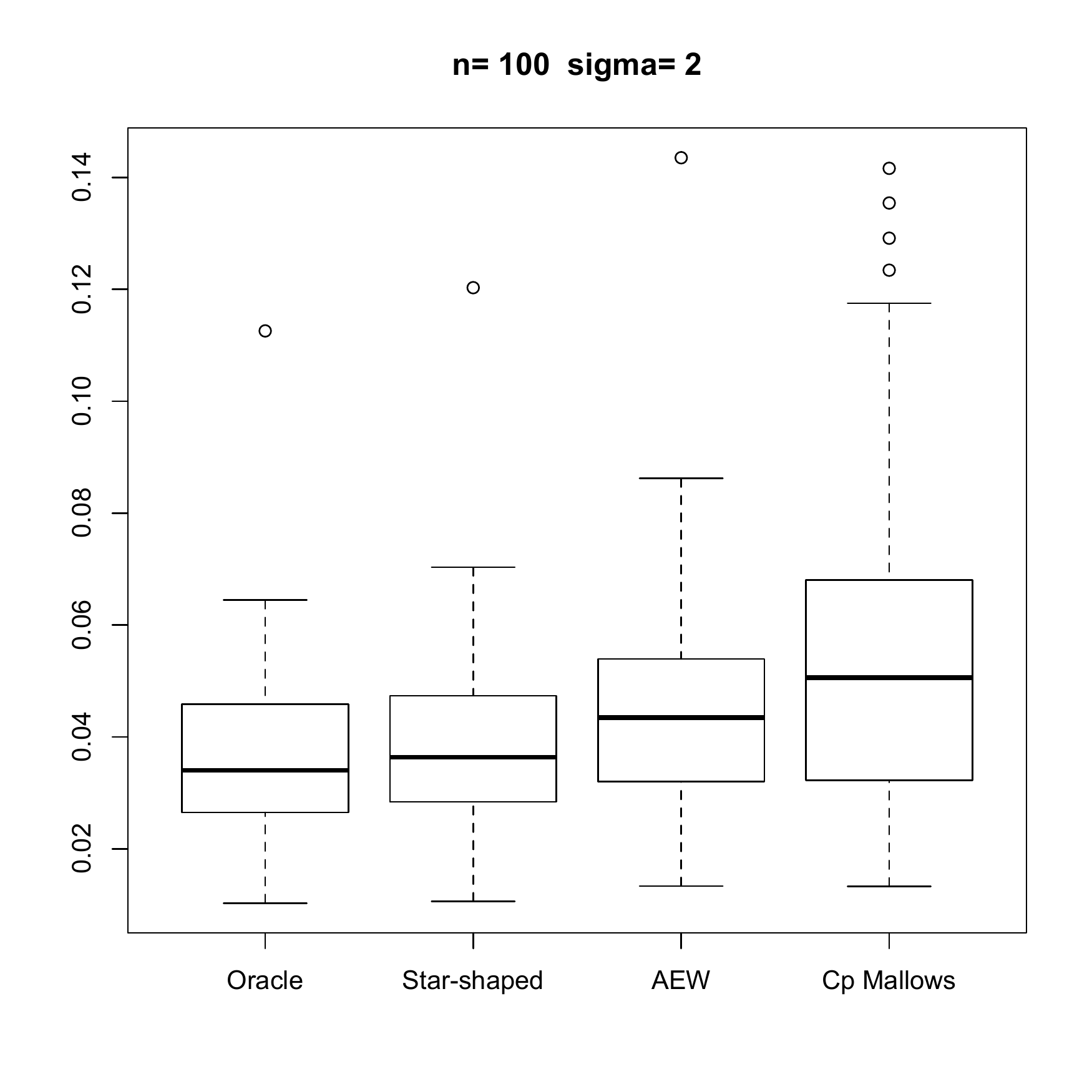}%
  \includegraphics[width=\figwidth,height=\figheight]{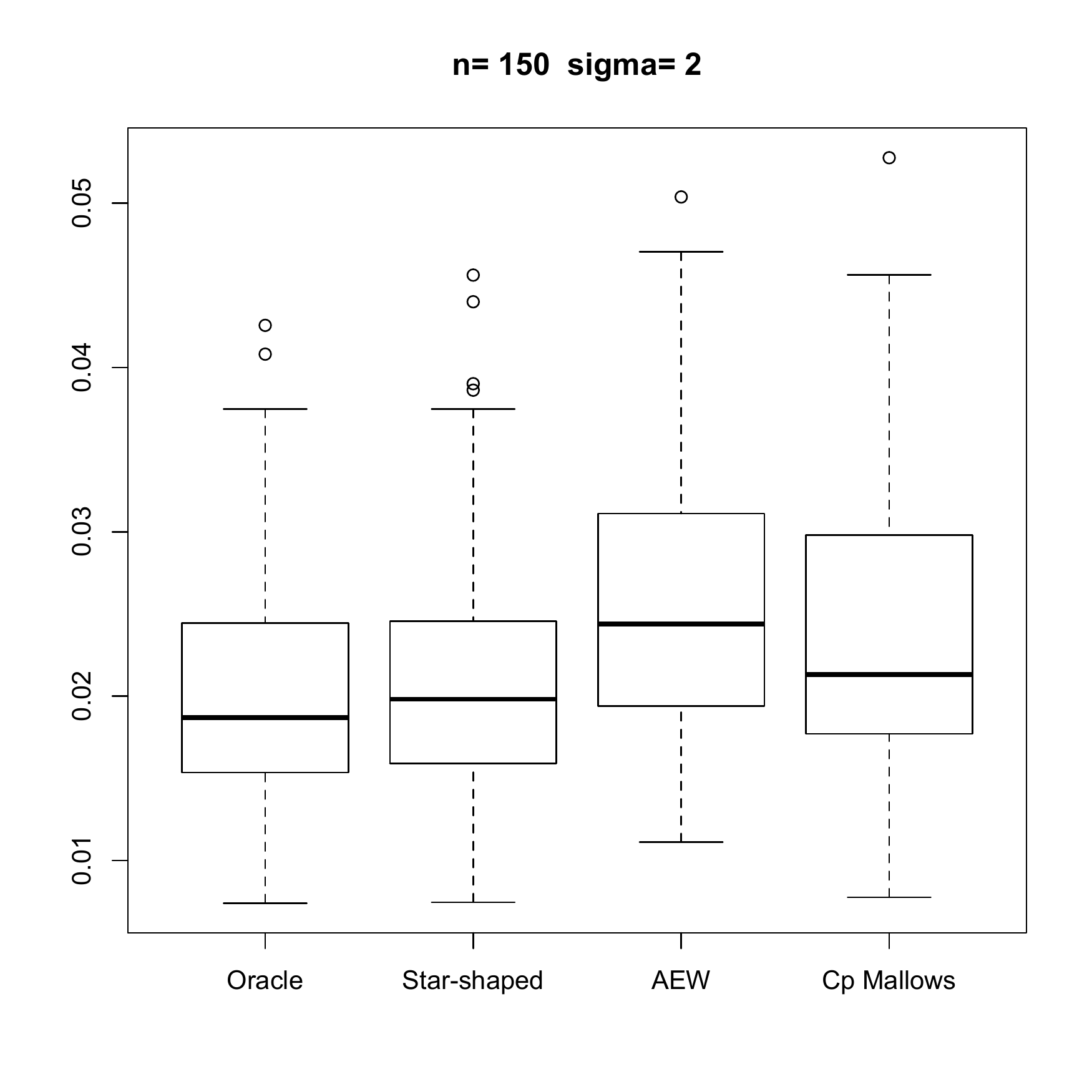} \\
  \includegraphics[width=\figwidth,height=\figheight]{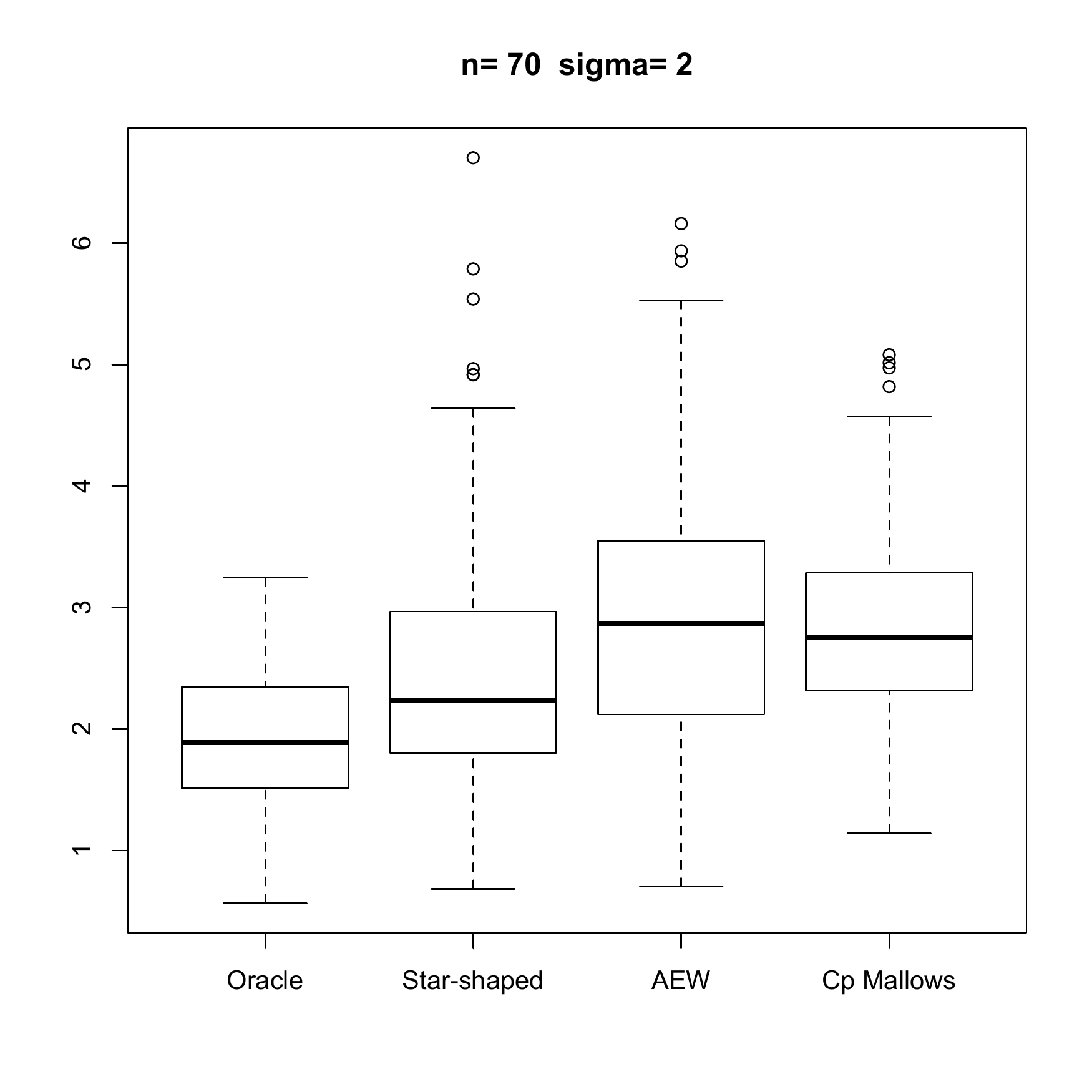}%
  \includegraphics[width=\figwidth,height=\figheight]{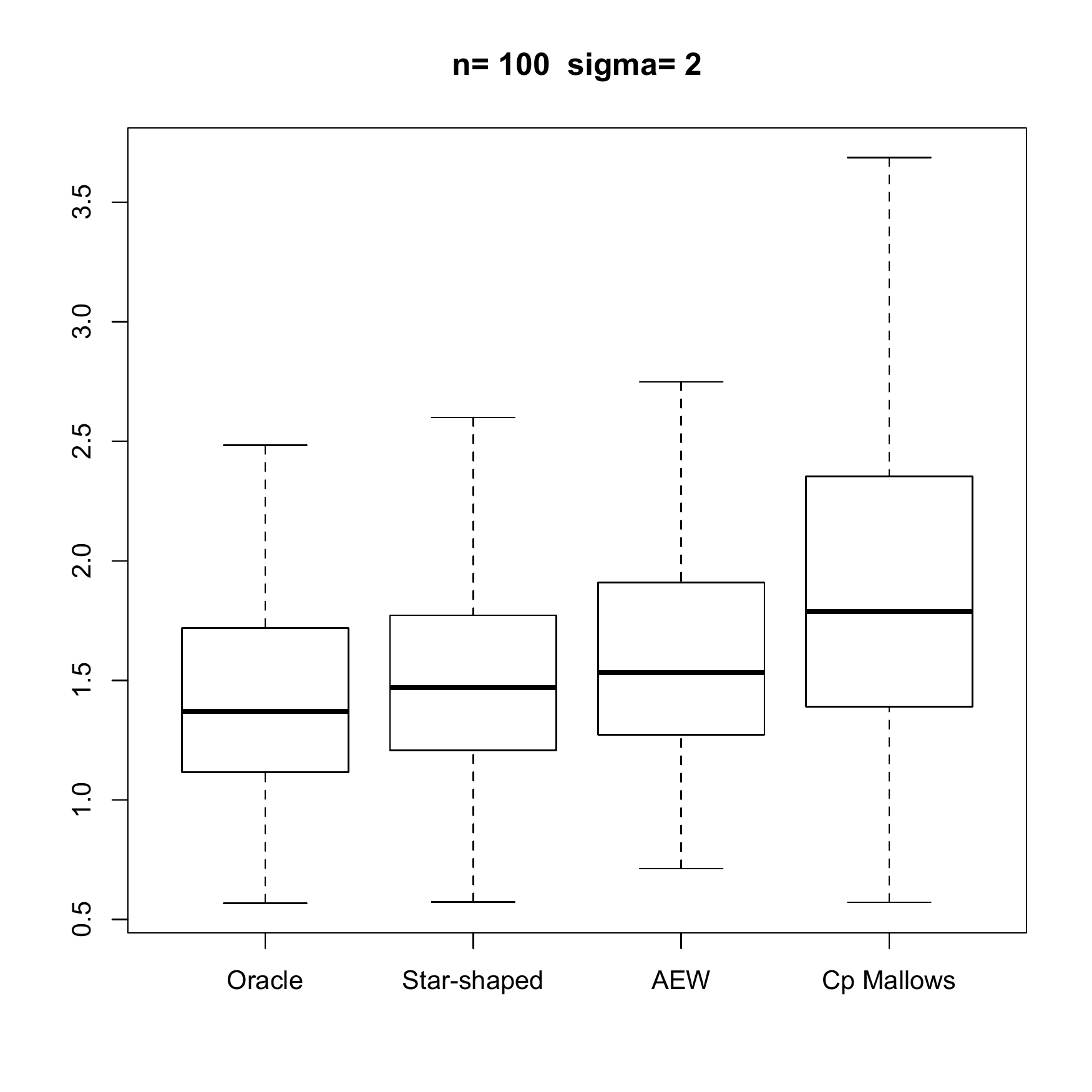}%
  \includegraphics[width=\figwidth,height=\figheight]{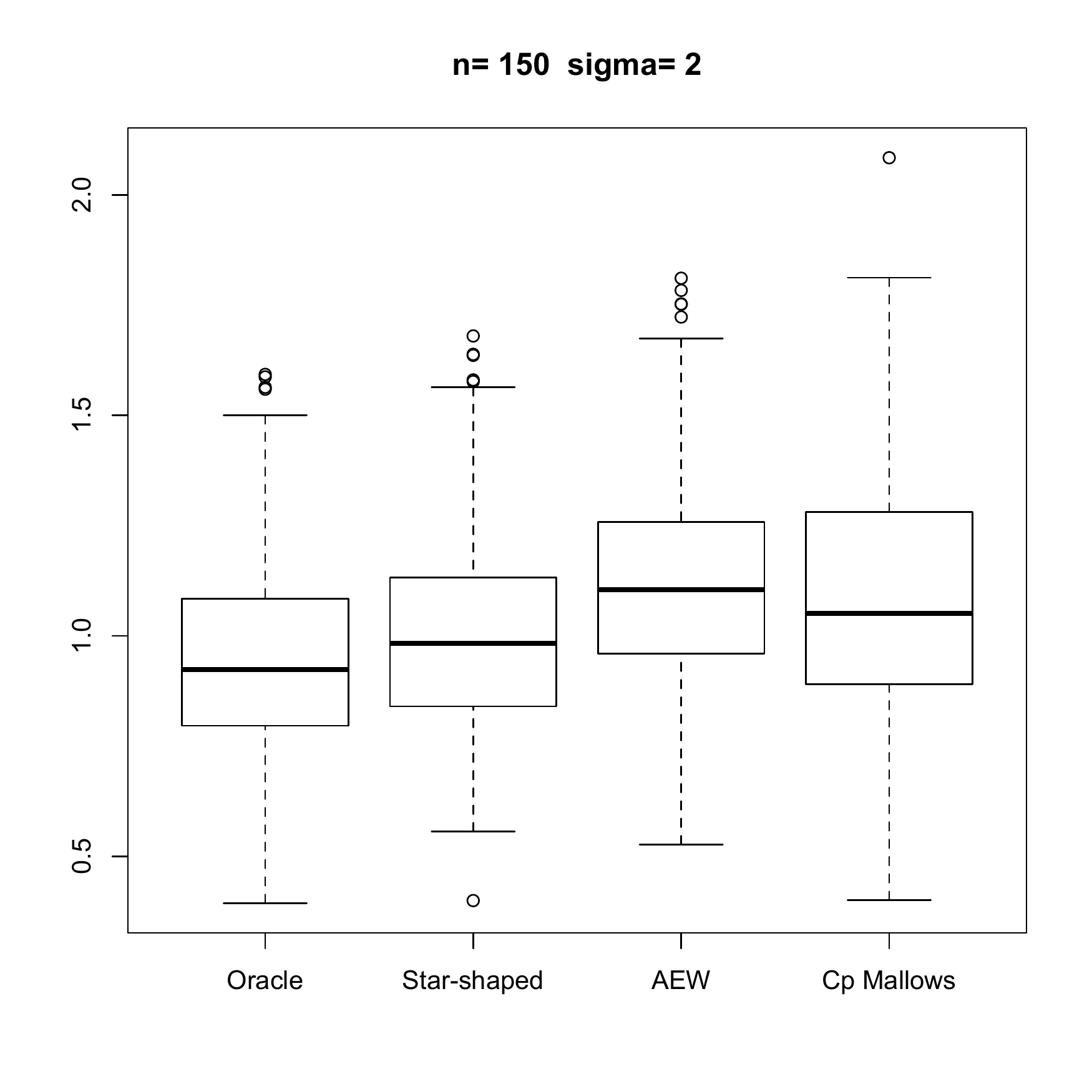}%
  \caption{Errors $| \hat \beta - \beta_0|_2$ (first row) and $|\bs
    X(\hat \beta - \beta_0)|_2$ (second row) for $\sigma=2$ and $n=70,
    100, 150$.}
  \label{fig:sigma2}
\end{figure}

\begin{figure}[htbp]
  \centering
  \includegraphics[width=\figwidth,height=\figheight]{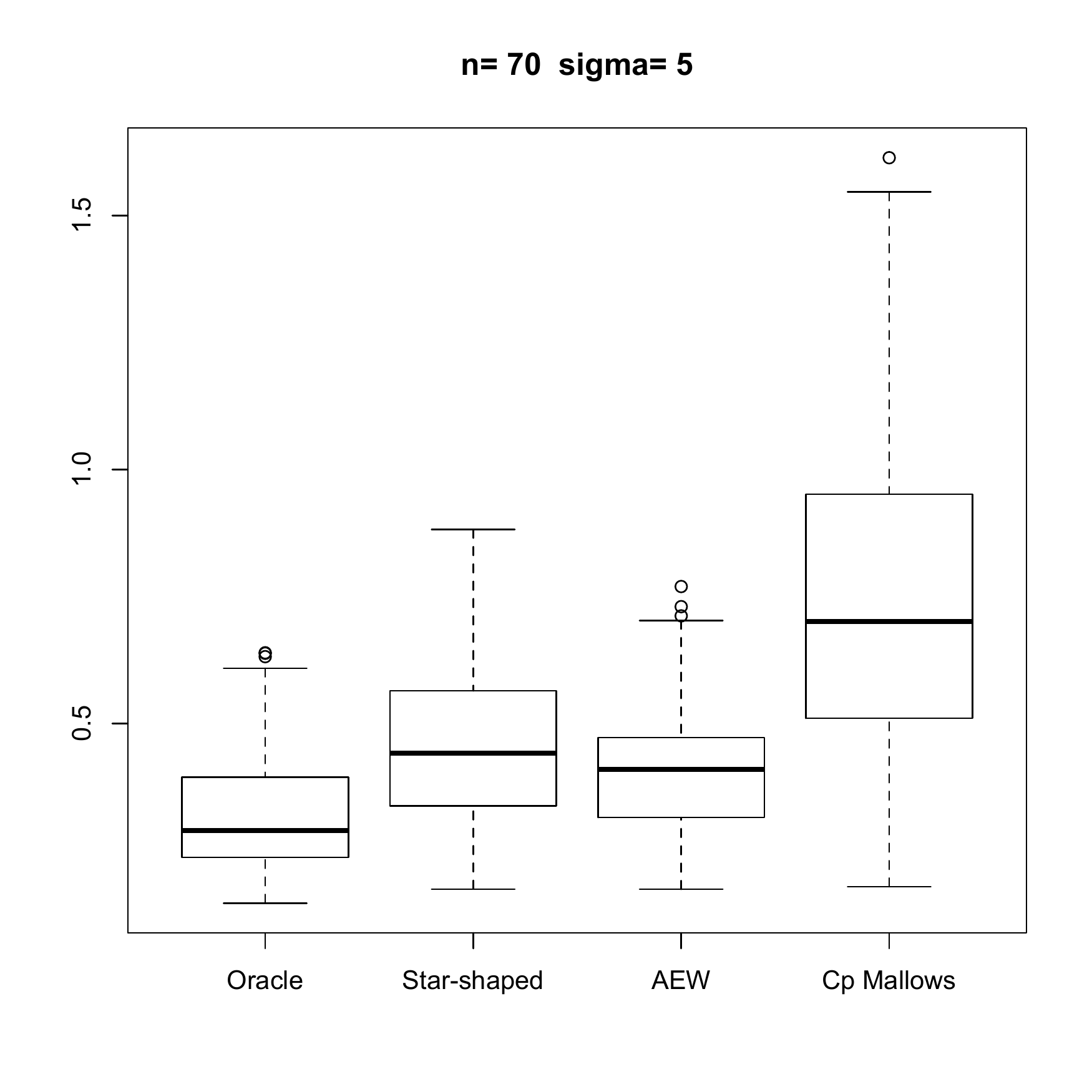}%
  \includegraphics[width=\figwidth,height=\figheight]{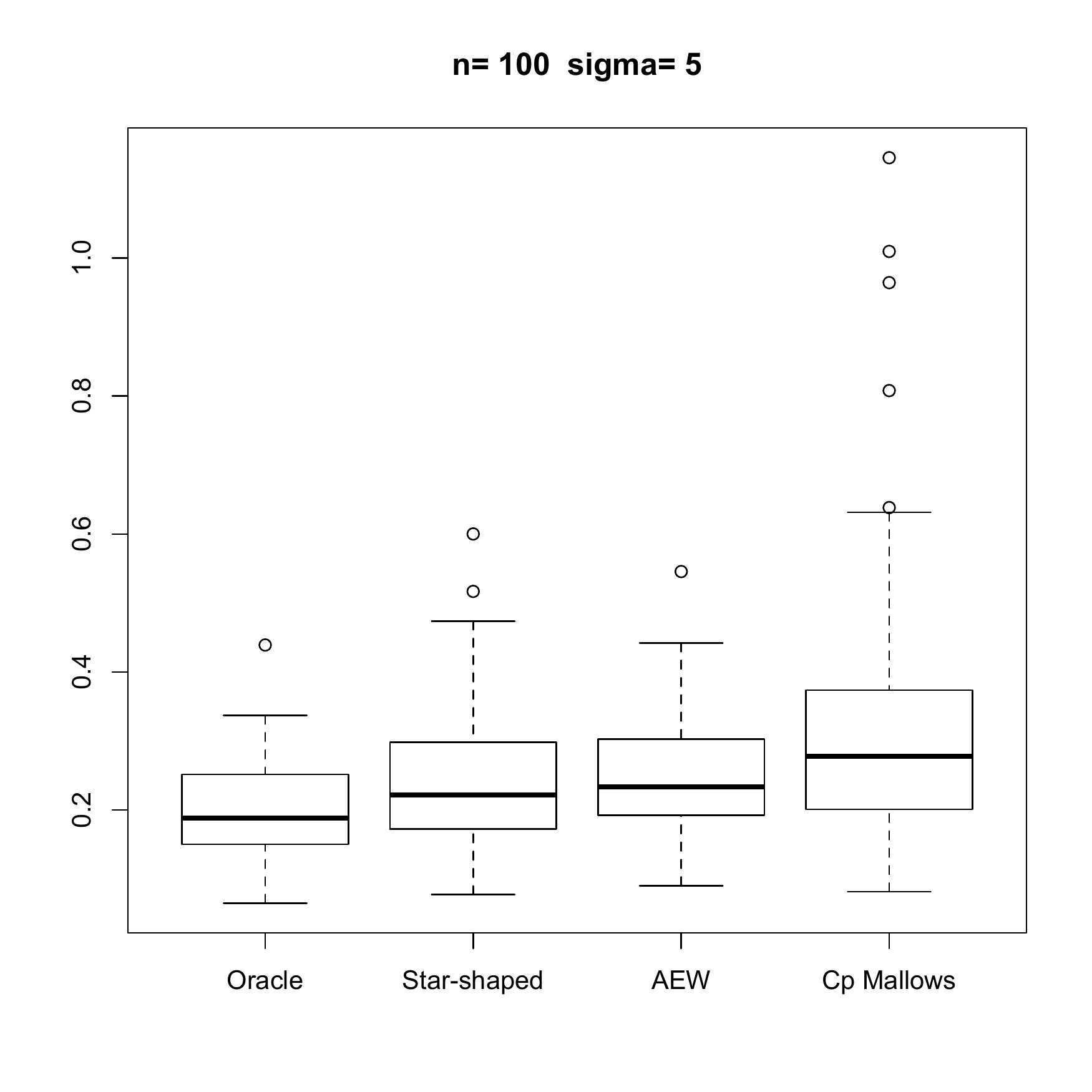}%
  \includegraphics[width=\figwidth,height=\figheight]{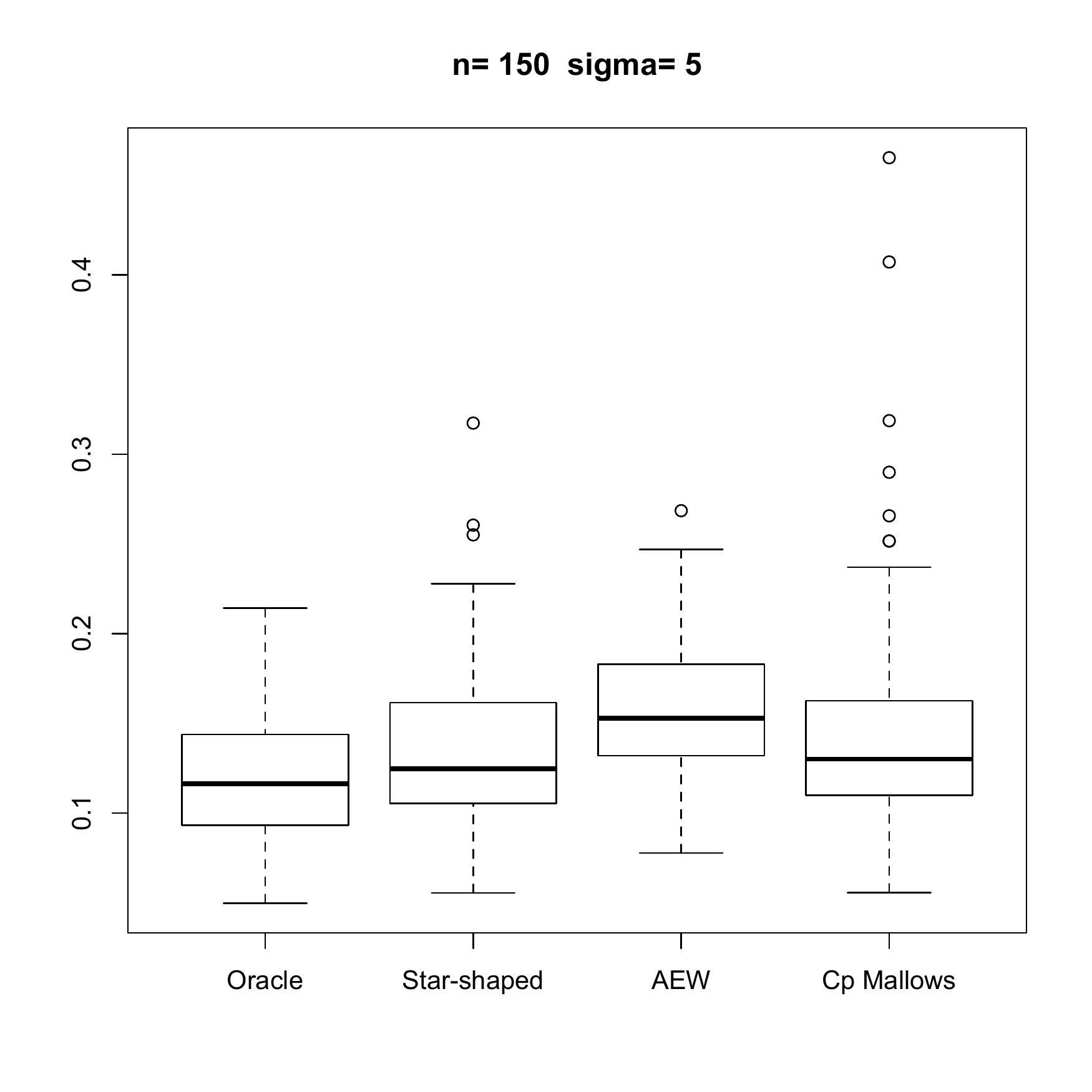} \\
  \includegraphics[width=\figwidth,height=\figheight]{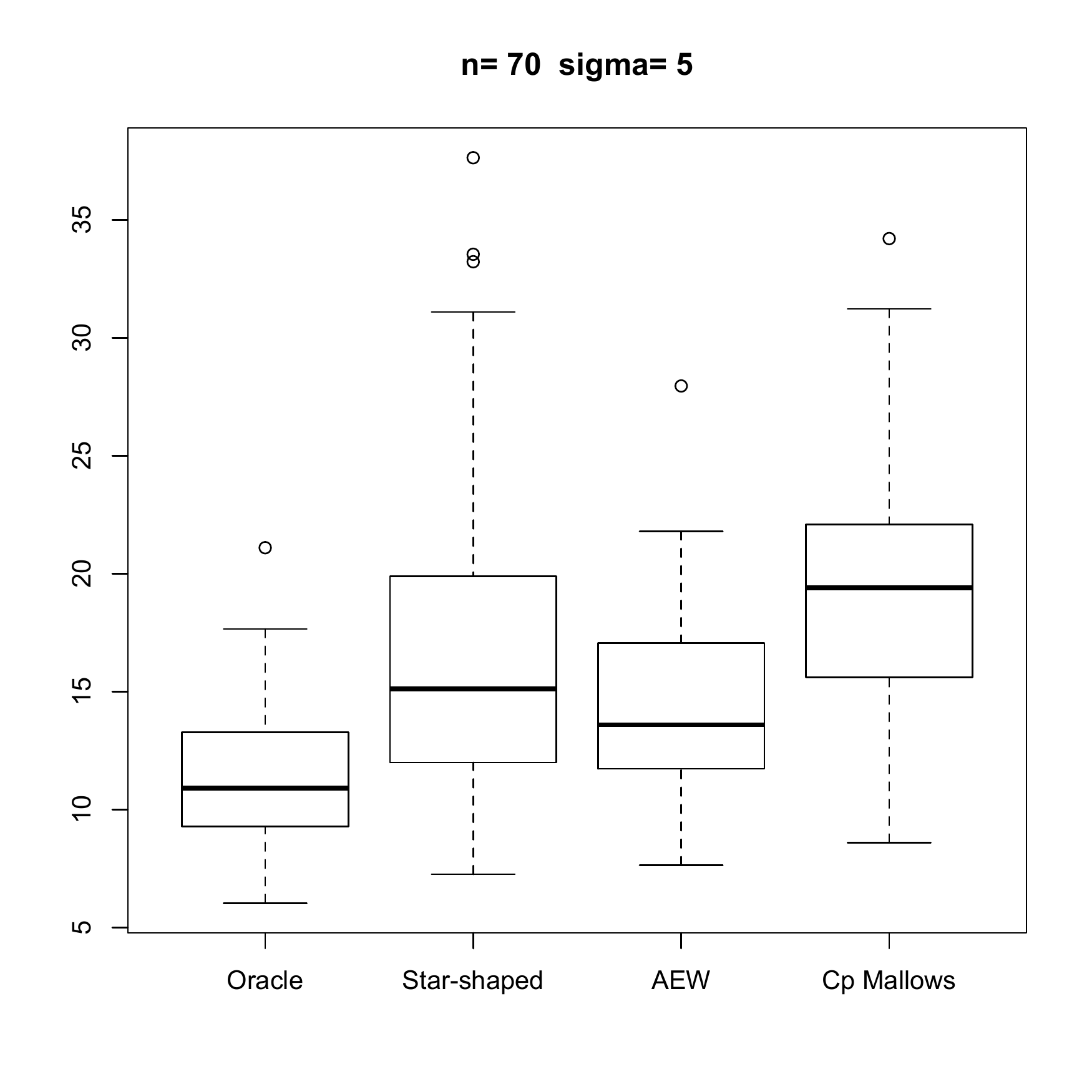}%
  \includegraphics[width=\figwidth,height=\figheight]{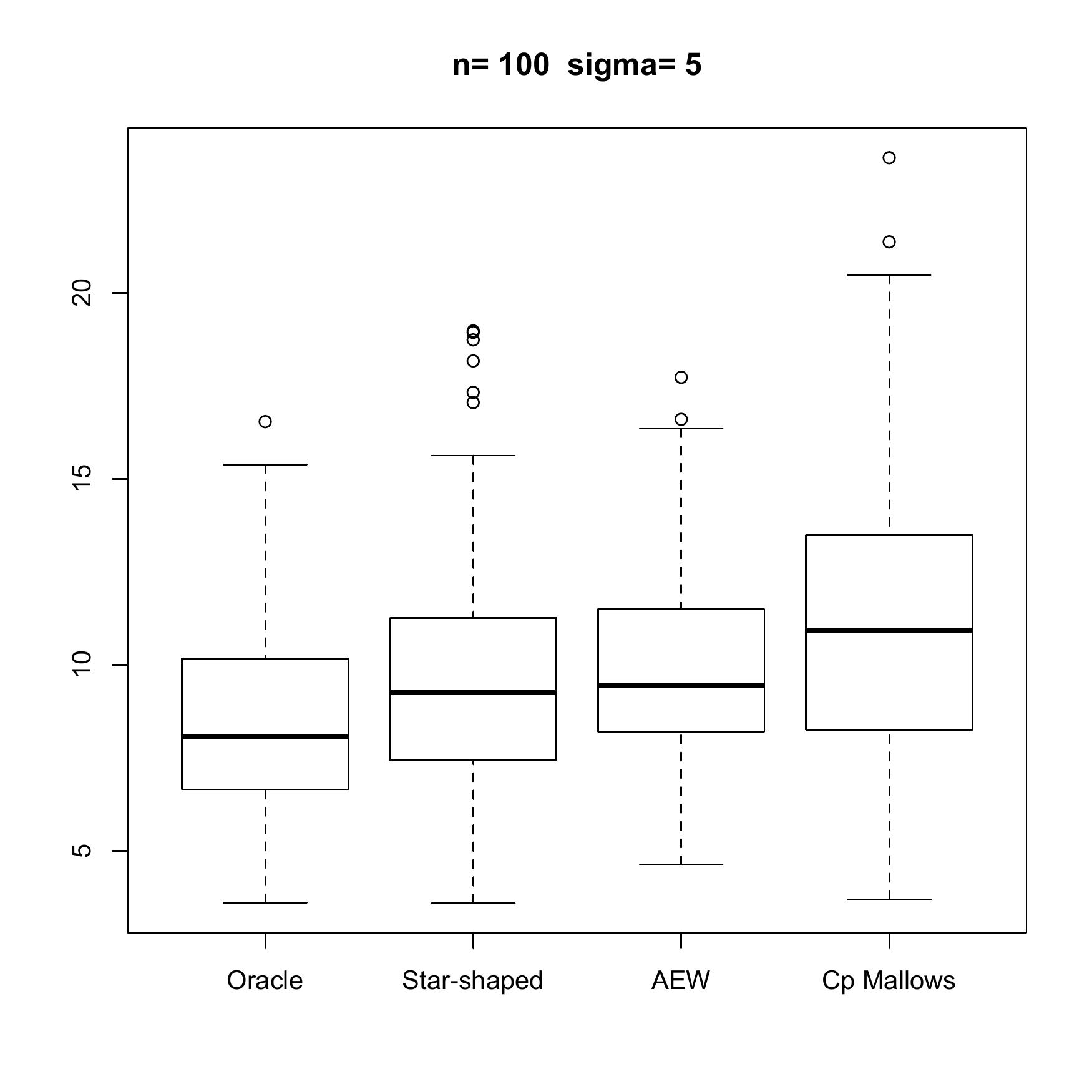}%
  \includegraphics[width=\figwidth,height=\figheight]{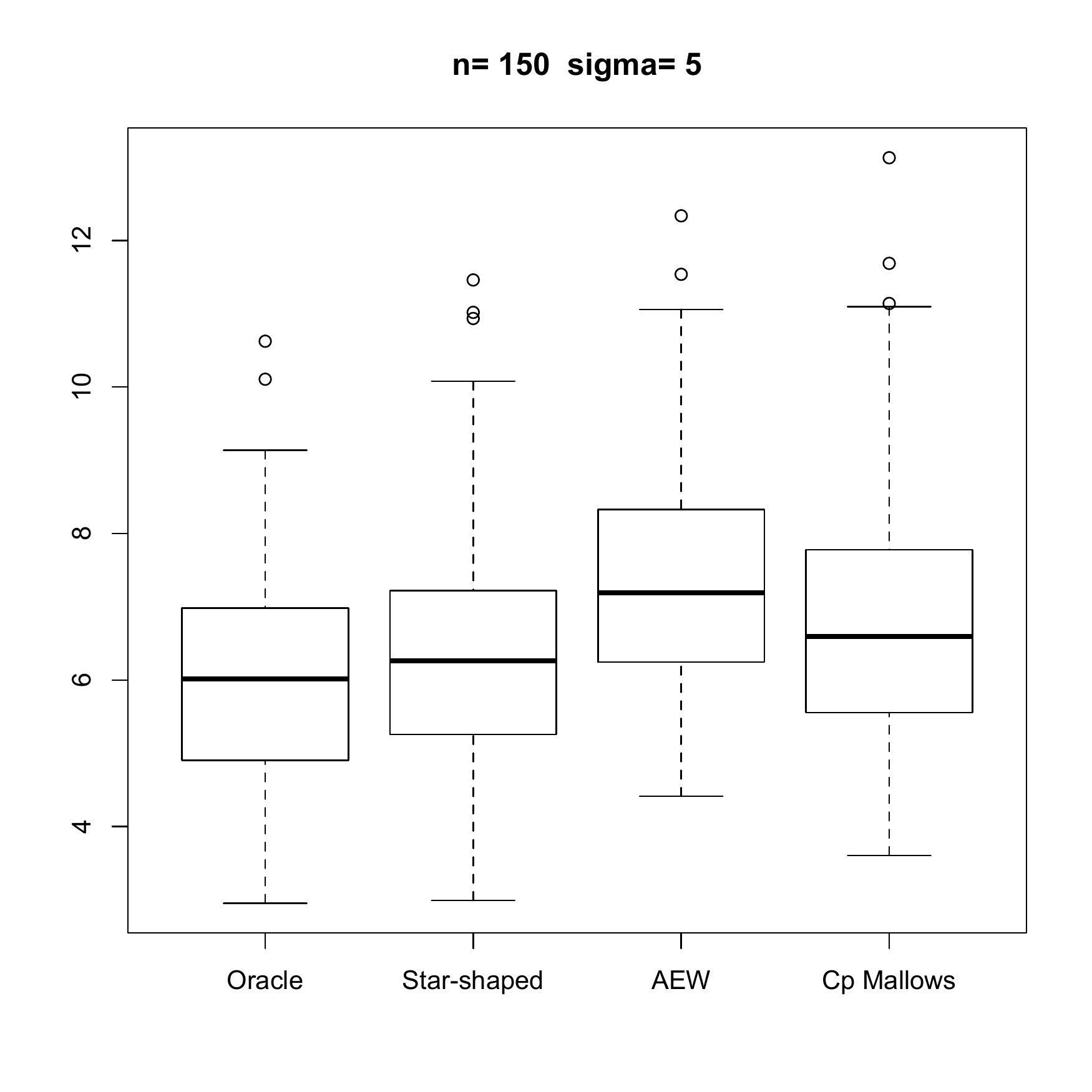}%
  \caption{Errors $|\hat \beta - \beta_0|_2$ (first row) and $|\bs
    X(\hat \beta - \beta_0)|_2$ (second row) for $\sigma=5$ and $n=70,
    100, 150$.}
  \label{fig:sigma5}
\end{figure}


\section{Proofs of the main results}
\label{sec:proof_main_results}

\subsection{Proof of Theorem~\ref{thm:aggregation}}

\begin{proof}[Proof of Theorem~\ref{thm:aggregation}]
  Let us prove the result in the $\psi_2$ case, the other case is
  similar. Fix $x>0$ and let $\FF$ be either \eqref{eq:convex_hull},
  \eqref{eq:segments} or \eqref{eq:starified}. Set $d :=
  \diam(\hat{F}_1, L_2(\mu))$. Consider the second half of the sample
  $D_{n, 2} = (X_i,Y_i)_{i=n+1}^{2n}$. By Corollary
  \ref{cor:high-prob-est} (see
  Appendix~\ref{sec:tools_empirical_process} below), with probability
  at least $1-4\exp(-x)$ (relative to $D_{n, 2}$), we have for every $f
  \in \FF$
  \begin{equation*}
    \Big| \frac{1}{n}\sum_{i=1+n}^{2n} {\cal L}_{\FF}(f)(X_i,Y_i) -
    \E\big({\cal L}_{\FF}(f)(X,Y) | D_{n, 1} \big) \Big| \leq c
    (\sigma_\varepsilon + b) \max(d \phi, b \phi^2 ),
  \end{equation*}
  where ${\cal L}_{\FF}(f)(X,Y) := (f(X) - Y)^2 - (f^{\FF}(X) - Y)^2$
  is the excess loss function relative to $\FF$, $f^{\FF}\in{\rm Arg}\min_{f\in\FF}R(f)$ and where $\phi =
  \sqrt{((\log M + x) \log n) / n}$. By definition of $\tilde{f}$, we
  have $\frac{1}{n} \sum_{i=n+1}^{2n} {\cal
    L}_{\FF}(\tilde{f})(X_i,Y_i) \leq 0$, so, on this event (relative to $D_{n,2}$)
  \begin{align}
    \label{eq:empirical}
    R(\tilde{f}) & \leq R(f^{\FF}) + \E \big({\cal L}_{\FF}
    (\tilde{f}) | D_{n, 1} \big) -\frac{1}{n} \sum_{i=n+1}^{2n} {\cal
      L}_{\FF} (\tilde{f})(X_i,Y_i) \\
    \nonumber &\leq R(f^{\FF}) + c (\sigma_{\varepsilon} + b) \max ( d
    \phi, b \phi^2 ) \\
    \nonumber & = R(f^F) + \Big(c (\sigma_{\varepsilon} + b) \max ( d
    \phi, b \phi^2
    ) - \big( R(f^F) - R(f^{\FF}) \big) \Big) \\
    \nonumber & =: R(f^F) + \beta,
  \end{align}
  and it remains to show that
  \begin{equation*}
    \beta \leq c_{b, \sigma_{\varepsilon}} \frac{(1 + x) \log M \log n}{n}.
  \end{equation*}
  When $\FF$ is given by~\eqref{eq:convex_hull} or
  \eqref{eq:segments}, the geometrical configuration is the same as in
  \cite{LM06}, so we skip the proof.
  Let us turn out to the situation where $\FF$ is given by
  \eqref{eq:starified}. Recall that $\hat f_{n, 1}$ is the ERM on
  $\hat F_1$ using $D_{n, 1}$. Consider $f_1$ such that $\norm{\hat
    f_{n, 1} - f_1}_{L^2(\mu)} = \max_{f \in \hat F_1} \norm{\hat
    f_{n, 1} - f}_{L^2(\mu)}$, and note that $\norm{\hat f_{n, 1} -
    f_1}_{L^2(\mu)} \leq d \leq 2 \norm{\hat f_{n, 1} -
    f_1}_{L^2(\mu)}$.  The mid-point $f_2 := (\hat f_{n, 1} + f_1) /
  2$ belongs to $\Star(\hat f_{n, 1}, \hat F_1)$. Using the
  parallelogram identity, we have for any $u, v \in L_2(\nu)$:
  \begin{equation*}
    \E_\nu \Big(\frac{u + v}{2} \Big)^2 \leq \frac{\E_\nu
      (u^2) + \E_\nu (v^2)}{2} - \frac{ \|u - v\|_{L_2(\nu)}^2}{4},
  \end{equation*}
  where for every $h \in L_2(\nu)$, $\E_\nu(h) = \E h(X,Y)$. In
  particular, for $u(X,Y) = \hat f_{n, 1} - Y$ and $v(X,Y) = f_1(X) -
  Y$, the mid-point is $(u(X,Y) + v(X,Y)) / 2 = f_2(X) - Y$. Hence,
  \begin{align*}
    R(f_2) & = \E (f_2(X) - Y)^2 = \E \Big( \frac{\hat f_{n, 1}(X) +
      f_1(X)}{2} - Y \Big)^2 \\
    & \leq \frac{1}{2} \E (\hat f_{n, 1}(X) - Y)^2 + \frac{1}{2} \E
    (f_1(X) - Y)^2 - \frac 14 \| f_{n, 1} - f_1 \|_{L_2(\mu)}^2 \\
    & \leq \frac{1}{2} R(\hat f_{n, 1}) +\frac{1}{2} R(f_1) -
    \frac{d^2}{16},
  \end{align*}
  where the expectations are taken conditioned on $D_{n, 1}$. By Lemma
  \ref{lemma:prop-random-set} (see
  Appendix~\ref{sec:tools_empirical_process} below), since $\hat f_{n,
    1}, f_1 \in \hat F_1$, we have
  \begin{equation*}
    \frac{1}{2} R(\hat f_{n, 1}) + \frac{1}{2} R(f_1) \leq R(f^F) + c
    (\sigma_{\varepsilon} + b)
    \max ( \phi d, b \phi^2 ),
  \end{equation*}
  and thus, since $f_2 \in \FF$
  \begin{equation*}
    R(f^{\FF}) \leq R(f_2) \leq R(f^F) + c (\sigma_{\varepsilon} + b)
    \max (\phi d, b \phi^2 ) - c d^2.
  \end{equation*}
  Therefore,
  \begin{align*}
    \beta & = c (\sigma_{\varepsilon} + b) \max ( d \phi, b \phi^2 )
    - \big( R(f^F) - R(f^{\FF}) \big) \\
    & \leq c (\sigma_{\varepsilon} + b) \max ( \phi d, b \phi^2 ) - c
    d^2.
  \end{align*}
  Finally, if $d \geq c_{\sigma_{\varepsilon}, b} \phi$ then $\beta
  \leq 0$, otherwise $\beta \leq c_{\sigma_{\varepsilon}, b} \phi^2$.
\end{proof}

\begin{proof}[Proof of Theorem \ref{TheoWeaknessERMRegression}]
 The dictionary $F_M$ is chosen so that we have, for any
  $j \in \{ 1, \ldots ,M-1 \}$
  \begin{equation*}
    \| f_j - f_0 \|_{L^2([0,1])}^2 = \frac{5 h^2}{2} + 1 \;\text{ and }\;
    \|f_M - f_0 \|_{L^2([0,1])}^2 = \frac{5h^2}{2} - h + 1.
  \end{equation*}
  Thus, we have
  \begin{equation*}
    \min_{j=1,\ldots,M} \|f_j - f_0 \|_{L^2([0,1])}^2 = \|f_M - f_0
    \|_{L^2([0,1])}^2 = \frac{5h^2}{2} -h + 1.
  \end{equation*}
  This geometrical setup for $F(\Lambda)$, which is a unfavourable
  setup for the ERM, is represented in Figure~\ref{fig:badsetup}. For
  \begin{equation*}
    \hat{f}_n := \tilde{f}_n^{\rm PERM} \in \argmin_{f \in F_M}
    \big(R_n(f) + \pen(f) \big),
  \end{equation*}
  where we take $R_n(f) = \frac{1}{n} \sum_{i=1}^n (Y_i-f(X_i))^2 =\|
  Y - f \|^2_n$, we have
  \begin{equation}
    \label{InegGaussian}
    \E \|\hat{f}_n - f_0 \|_{L^2([0,1])}^2 =
    \min_{j=1,\ldots,M} \|f_j - f_0 \|_{L^2([0,1])}^2 + h
    \P[\hat{f}_n\neq f_M].
  \end{equation}
  Now, we upper bound $\P[ \hat{f}_n= f_M]$.  We consider the dyadic
  decomposition of the design variable $X$:
    \begin{equation}
    \label{EquaDyadicRegression}
    X = \sum_{k = 1}^{+\infty} X^{(k)} 2^{-k},
  \end{equation}
  where $(X^{(k)} : k \geq 1)$ is a sequence of i.i.d. random
  variables following a Bernoulli $\cB(1/2,1)$ with parameter $1/2$
  (because $X$ is uniformly distributed on $[0,1]$).  If we define
  \begin{equation*}
    N_j := \frac{1}{\sqrt{n}} \sum_{i=1}^n\zeta_i^{(j)}
    \varepsilon_i \text{ and } \zeta_i^{(j)} := 2X_i^{(j)}-1,
  \end{equation*}
  we have by the definition of $h$ and since $\zeta_i^{(j)} \in \{ -1,
  1\}$:
  \begin{align*}
    \frac{\sqrt{n}}{2 \sigma} (\norm{Y - f_M}_n^2 &- \norm{Y -
      f_j}_n^2) \\
    & = N_j - N_M + \frac{h}{2 \sigma \sqrt{n}} \sum_{i=1}^n
    (\zeta_i^{(j)} \zeta_i^{(M)} + 3(\zeta_i^{(j)} - \zeta_i^{(M)}) -
    1) \\
    &\geq N_j - N_M - \frac{4C}{\sigma} \sqrt{\log M}.
  \end{align*}
  This entails, for $\bar N_{M-1} := \max_{1 \leq j \leq N-1} N_j$,
  that
  \begin{align*}
    \P[ \hat{f}_n= f_M] &= P \Big[ \bigcap_{j=1}^{M-1} \Big\{ \norm{Y -
      f_M}_n^2 - \norm{Y - f_j}_n^2 \leq \pen(f_j) - \pen(f_M) \Big\}
    \Big] \\
    &\leq \P\Big[ N_M \geq \bar N_{M-1} - \frac{6C}{\sigma} \sqrt{\log
      M} \Big].
  \end{align*}
  It is easy to check that $N_1, \ldots, N_M$ are $M$ normalized
  standard gaussian random variables uncorrelated (but dependent). We
  denote by $\boldsymbol{\zeta}$ the family of Rademacher variables
  $(\zeta_i^{(j)} : i=1,\ldots,n ; j=1,\ldots,M)$. We have for any
  $6C/\sigma <\gamma< (2\sqrt{2}c^*)^{-1}$ ($c^*$ is the ``Sudakov
  constant'', see Theorem~\ref{TheoSudakov}),
  \begin{align}
    \label{EquaSudakov}
    \P[\hat{f}_n = f_M] &\leq \E \Big[ \P\Big( N_M \geq \bar N_{M-1} -
    \frac{6C}{\sigma}\sqrt{\log M} \Big| \boldsymbol{\zeta} \Big)
    \Big] \nonumber \\
    &\leq \P \big[ N_M \geq - \gamma \sqrt{\log M}
    +  \E(\bar N_{M-1} | \boldsymbol{\zeta} ) \big] \\
    &+ \E \Big[ \P\Big\{ \E( \bar N_{M-1} | \boldsymbol{\zeta} ) -
    \bar N_{M-1} \geq (\gamma - \frac{6C}{\sigma}) \sqrt{\log M} \Big|
    \boldsymbol{\zeta} \Big\} \Big]. \nonumber
  \end{align}
  Conditionally to $\boldsymbol{\zeta}$, the vector
  $(N_1,\ldots,N_{M-1})$ is a linear transform of the Gaussian vector
  $(\varepsilon_1, \ldots, \varepsilon_n)$. Hence, conditionally to
  $\boldsymbol{\zeta}$, $(N_1,\ldots,N_{M-1})$ is a gaussian
  vector. Thus, we can use a standard deviation result for the
  supremum of Gaussian random vectors (see for
  instance~\cite{massart03}, Chapter~3.2.4), which leads to the
  following inequality for the second term of the RHS
  in~\eqref{EquaSudakov}:
  \begin{align*}
    \P \Big\{ \E( \bar N_{M-1} | \boldsymbol{\zeta} ) - \bar N_{M-1}
    \geq (\gamma &- \frac{6C}{\sigma}) \sqrt{\log M} \Big|
    \boldsymbol{\zeta}
    \Big\} \\
    &\leq \exp(-(3C/\sigma-\gamma/2)^2\log M).
  \end{align*}
  Remark that we used $\E[ N_j^2 | \boldsymbol{\zeta}] = 1$ for any $j
  = 1, \ldots, M-1$. For the first term in the RHS
  of~\eqref{EquaSudakov}, we have
  \begin{align}
    \label{EquaIerTermSudakov}
    \P &\Big [N_M \geq - \gamma \sqrt{\log M}
    + \E( \bar N_{M-1} | \boldsymbol{\zeta} ) \Big] \nonumber\\
    &\leq \P \Big[N_M \geq - 2 \gamma \sqrt{\log M}
    + \E(\bar N_{M-1}) \Big] \\
    &+\P \Big[ - \gamma\sqrt{\log M} + \E(\bar N_{M-1}) \geq \E(\bar
    N_{M-1} | \boldsymbol{\zeta}) \Big]. \nonumber
  \end{align}
  Next, we use Sudakov's Theorem (cf. Theorem \ref{TheoSudakov} in
  Appendix~\ref{sec:appendix_proba}) to lower bound $\E( \bar
  N_{M-1})$. Since $(N_1,\ldots,N_{M-1})$ is, conditionally to
  $\boldsymbol{\zeta}$, a Gaussian vector and since for any $1 \leq j
  \neq k \leq M$ we have
  \begin{equation*}
    \E[(N_k-N_j)^2 | \boldsymbol{\zeta}] = \frac{1}{n}
    \sum_{i=1}^n (\zeta_i^{(k)} - \zeta_i^{(j)})^2
  \end{equation*}
  then, according to Sudakov's minoration
  (cf. Theorem~\ref{TheoSudakov} in the Appendix), there exits an
  absolute constant $c^* > 0$ such that
  \begin{equation*}
    c^* \E[\bar N_{M-1} | \boldsymbol{\zeta}] \geq
    \min_{1 \leq j \neq k \leq M-1} \Big(\frac{1}{n}\sum_{i=1}^n
    (\zeta_i^{(k)} - \zeta_i^{(j)})^2\Big)^{1/2} \sqrt{\log M}.
  \end{equation*}
  Thus, we have
  \begin{align*}
    \label{EquaSudak3}
    c^* \E[\bar N_{M-1}] &\geq \E\Big[ \min_{j \neq k}
    \Big(\frac{1}{n} \sum_{i=1}^n (\zeta_i^{(k)} - \zeta_i^{(j)})^2
    \Big)^{1/2} \Big] \sqrt{\log M} \\
    &\geq \sqrt{2} \Big(1 - \E \Big[ \max_{j\neq k} \frac{1}{n}
    \sum_{i=1}^n \zeta_i^{(k)} \zeta_i^{(j)} \Big] \Big) \sqrt{\log
      M},
  \end{align*}
  where we used the fact that $\sqrt{x} \geq x/\sqrt{2}, \forall x \in
  [0,2]$.
  Besides, using Hoeffding's inequality we have $\E[\exp(s
  \xi^{(j,k)})] \leq \exp(s^2/(2n))$ for any $s > 0$, where
  $\xi^{(j,k)} := n^{-1} \sum_{i=1}^n \zeta_i^{(k)} \zeta_i^{(j)}$.
  Then, using a maximal inequality (cf.  Theorem~\ref{TheoMaxConcIneq}
  in Appendix~\ref{sec:appendix_proba}) and since $n^{-1}
  \log[(M-1)(M-2)] \leq 1/4$, we have
  \begin{equation}
    \label{EquaSudakFinal}
    \E\Big[\max_{j\neq k} \frac{1}{n} \sum_{i=1}^n
    \zeta_i^{(k)} \zeta_i^{(j)} \Big] \leq
    \Big(\frac{1}{n} \log[(M-1)(M-2)] \Big)^{1/2} \leq
    \frac{1}{2}.
  \end{equation}
  This entails
  \begin{equation*}
    c^* E[ \bar N_{M-1} ] \geq \Big(\frac{\log M}{2} \Big)^{1/2}.
  \end{equation*}
  Thus, using this inequality in the first RHS
  of~\eqref{EquaIerTermSudakov} and the usual inequality on the tail
  of a Gaussian random variable ($N_M$ is standard Gaussian), we
  obtain:
  \begin{align}
    \label{EquaFirstTerm}
    \P\Big[N_M \geq &-2\gamma \sqrt{\log M} + \E(\bar N_{M-1}) \Big]
    \leq \P\Big[ N_M \geq ((c^*\sqrt{2})^{-1}-2\gamma)
    \sqrt{\log M}\Big]\nonumber\\
    &\leq \P\Big[N_M \geq ((c^*\sqrt{2})^{-1}-2\gamma) \sqrt{\log
      M}\Big]\\
    &\leq \exp\Big(-((c^*\sqrt{2})^{-1}-2\gamma)^2(\log
    M)/2\Big).\nonumber
  \end{align}
  Remark that we used $2\sqrt{2}c^* \gamma < 1$. For the second term
  in (\ref{EquaIerTermSudakov}), we apply the concentration inequality
  of Theorem \ref{TheoEinmahlMasson} to the non-negative random
  variable $\E[\bar N_{M-1}|\boldsymbol{\zeta}]$. We first have to
  control the second moment of this variable. We know that,
  conditionally to $\boldsymbol{\zeta}$,
  $N_j|\boldsymbol{\zeta}\sim\cN(0,1)$ thus,
  $N_j|\boldsymbol{\zeta}\in L_{\psi_2}$ (for more details on Orlicz
  norm, we refer the reader to~\cite{vdVW:96}). Thus,
  \begin{equation*}
    \norm{\max_{1\leq j\leq M-1} N_j | \boldsymbol{\zeta}}_{\psi_2}\leq
    K \psi_2^{-1}(M)\max_{1\leq j\leq M-1}\norm{N_j|\boldsymbol{\zeta}}_{\psi_2}
  \end{equation*}
  (cf. Lemma 2.2.2 in \cite{vdVW:96}). Since
  $\norm{N_j|\boldsymbol{\zeta}}_{\psi_2}^2=1$, we have
  $\norm{\max_{1\leq j\leq M-1} N_j|\boldsymbol{\zeta}}_{\psi_2}\leq K
  \sqrt{\log M}$. In particular, we have $\E\big[\max_{1\leq j\leq
    M-1} N_j^2|\boldsymbol{\zeta} \big]\leq K\log M$ and so
  $\E\big(\E[\bar N_{M-1}|\boldsymbol{\zeta}]\big)^2\leq K\log
  M$. Then, Theorem \ref{TheoEinmahlMasson} provides
  \begin{equation}
    \label{SecondTermEquaSuda}
    \P\Big[ -\gamma \sqrt{\log M} + \E[\bar N_{M-1}] \geq \E[\bar
    N_{M-1}|\boldsymbol{\zeta}]\Big]\leq
    \exp(-\gamma^2/c_0),
  \end{equation}
  where $c_0$ is an absolute constant.

  Finally, combining (\ref{EquaSudakov}), (\ref{EquaFirstTerm}),
  (\ref{EquaIerTermSudakov}), (\ref{SecondTermEquaSuda}) in the
  initial inequality (\ref{EquaSudakov}), we obtain
  \begin{align*}
    \P[\hat{f}_n= f_M] & \leq \exp(-(3C/\sigma-\gamma)^2\log M) \\
    &+ \exp\Big(-((c^*\sqrt{2})^{-1}-2\gamma)^2(\log M)/2\Big)+
    \exp(-\gamma^2/c_0).
  \end{align*}
  Take $\gamma = (12 \sqrt{2}c^*)^{-1}$. It is easy to find an integer
  $M_0(\sigma)$ depending only on $\sigma$ such that for any $M \geq
  M_0$, we have $\P[\hat{f}_n= f_M]\leq c_1<1$, where $c_1$ is an
  absolute constant.  We complete the proof by using this last result
  in (\ref{InegGaussian}).
\end{proof}

\begin{proof}[Proof of Theorem~\ref{thm:adaptive}]
  Recall that we use the sample $D_{n, 1}$ to compute the family $F =
  \{ -b \vee \bar f_{\bs s} \wedge b : \bs s \in \bs S_n \}$ of PERM,
  which has cardinality $c (\log n)^d$, and the sample $D_{n, 2}$ to
  compute the weights of the aggregate $\tilde f$, see
  Definition~\ref{def:adaptive-est}. Recall also that there is $\bs
  s_0 = (s_{0, 1}, \ldots, s_{0, d})\in \bs S$ such that $f_0 \in
  B_{p, \infty}^{{\bs s}_0}$, and denote $r_0 = |f_0|_{B_{p,
      \infty}^{\bs s_0}}$. Take $\bs s_* = (s_{*, 1}, \ldots, s_{*,
    d}) \in \bs S_n$ such that $s_{*, j} \leq s_{0, j} \leq s_{*, j} +
  (\log n)^{-1}$ for all $j = 1, \ldots, d$. Remark that for this
  choice, one has $B_{p, \infty}^{\bs s_0} \subset B_{p, \infty}^{\bs
    s_*}$ and $n^{-2 \bar {\bs s}_* / (2 \bar {\bs s}_* + d)} \leq
  e^{d/2} n^{-2 \bar {\bs s}_0 / (2 \bar {\bs s}_0 + d)}$. By a
  substraction of $P_{\ell_{f_0}}$ on both sides of the oracle
  inequality stated in Theorem~\ref{thm:aggregation}, and since
  $\norm{f_0}_\infty \leq b$, we can find an event $A_{x, n, 2}$
  satisfying $\nu^{2n}(A_{x, n, 2}) \geq 1 - 2 e^{-x}$ and on which:
  \begin{equation*}
    \norm{\tilde f - f_0}_{L^2(\mu)}^2 \leq \E( \ell_{\bar f_{\bs
        s_*}} - \ell_{f_0} | D_{n, 1}) + c\frac{(1 + x) \log \log n}{n}.
  \end{equation*}
  Now, using Theorem~\ref{thm:PERM}, we can find an event $A_{x, n,
    1}$ satisfying $\nu^{2n}(A_{x, n, 1}) \geq 1 - 2 e^{-x}$ on which:
  \begin{align*}
    \E( \ell_{\bar f_{\bs s_*}} - \ell_{f_0} | D_{n, 1}) &\leq \inf_{f
      : |f|_{B_{p, \infty}^{\bs s_*}} \leq r_0 } \E( \ell_f -
    \ell_{f_0}) + c_1 r_0^2 n^{-\frac{-2 \bs s_*}{2 \bs s_* + d}} \\
    &+ \frac{c_2 (1 + r_0^2)}{n} \big(x + \log(\frac{\pi^2}{6}) +
    \log( 1 + c_3 n + \log r_0 ) \big) \Big\} \\
    &\leq c_1' r_0^2 n^{-\frac{-2 \bs s_0}{2 \bs s_0 + d}} + \frac{c_2
      (1 + r_0^2)}{n} \big(x + \log(\frac{\pi^2}{6}) + \log( 1 + c_3 n
    + \log r_0 ) \big) \Big\},
  \end{align*}
  where we used the fact that $|f_0|_{B_{p, \infty}^{\bs s_*}} \leq
  |f_0|_{B_{p, \infty}^{\bs s_0}} \leq r_0$. This concludes the proof
  of Theorem~\ref{thm:adaptive}, since $\nu^{2n}(A_{x, n, 1} \cap
  A_{x, n, 2}) \geq 1 - 4 e^{-x}$.
\end{proof}

\appendix

\section{Tools from empirical process theory}
\label{sec:tools_empirical_process}

\subsection{Useful results from literature}

The following Theorem is a Talagrand's type concentration inequality
(see \cite{MR1419006}) for a class of unbounded functions.

\begin{theorem}[Theorem~4, \cite{MR2424985}]
  \label{thm:adamczak}
  Assume that $X,X_1, \ldots, X_n$ are independent random variables
  and $F$ is a countable set of functions such that $\E f(X) =
  0,\forall f\in F$ and, for some $\alpha \in (0, 1]$, $\norm{\sup_{f
      \in F} f(X) }_{\psi_\alpha} < +\infty$. Define
  \begin{equation*}
    Z := \sup_{f \in F} \Big| \frac 1n \sum_{i=1}^n f(X_i) \Big|
  \end{equation*} and 
  \begin{equation*}
    \sigma^2 = \sup_{f \in F} \E f(X)^2 \mbox{ and } b :=
    \frac{\norm{\max_{i=1, \ldots, n} \sup_{f \in F} | f(X_i) |
      }_{\psi_\alpha}}{n^{1 - 1 / \alpha}}.
  \end{equation*}
  Then, for any $\eta \in (0, 1)$ and $\delta > 0$, there is $c =
  c_{\alpha, \eta, \delta}$ such that for any $x > 0$:
  \begin{align*}
    \P\Big[ Z \geq (1 + \eta) \E Z + \sigma \sqrt{2 (1 + \delta) \frac
      xn} + c b \Big(\frac x n\Big)^{1 / \alpha} \Big] &\leq 4 e^{-x} \\
    \P\Big[ Z \leq (1 - \eta) \E Z - \sigma \sqrt{2 (1 + \delta) \frac
      xn} - c b \Big(\frac x n\Big)^{1 / \alpha} \Big] &\leq 4 e^{-x}.
  \end{align*}
\end{theorem}

\subsection{Some probabilistic tools}
\label{sec:appendix_proba}

For the first Theorem, we refer to \cite{EM:96}. The two following
Theorems can be found, for instance, in
\cite{massart03,vdVW:96,ledoux_talagrand91}.

\begin{theorem}[Einmahl and Masson (1996)]
  \label{TheoEinmahlMasson}
  Let $Z_1,\ldots,Z_n$ be $n$ independent non-negative random
  variables such that $\E[Z_i^2]\leq \sigma^2,\forall i=1, \ldots, n$.
  Then, we have, for any $\delta > 0$,
  \begin{equation*}
    \P \Big[\sum_{i=1}^n Z_i - \E[Z_i] \leq -n \delta \Big]
    \leq \exp\Big(-\frac{n \delta^2}{2\sigma^2} \Big).
  \end{equation*}
\end{theorem}

\begin{theorem}[Sudakov]
  \label{TheoSudakov}
  There exists an absolute constant $c^*>0$ such that for any integer
  $M$, any centered gaussian vector $X = (X_1,\ldots,X_M)$ in
  $\mathbb{R}^M$, we have,
  \begin{equation*}
    c^* \E[\max_{1\leq j\leq M}X_j] \geq \varepsilon \sqrt{\log M},
  \end{equation*}
  where $\varepsilon := \min \Big\{ \sqrt{\E[(X_i-X_j)^2]} : i \neq j
  \in \{1, \ldots, M\} \Big\}$.
\end{theorem}

\begin{theorem}[Maximal inequality]
  \label{TheoMaxConcIneq}
  Let $Y_1, \ldots, Y_M$ be $M$ random variables satisfying
  $\E[\exp(sY_j)] \leq \exp((s^2\sigma^2)/2)$ for any integer $j$ and
  any $s>0$. Then, we have
  \begin{equation*}
    \E[ \max_{1 \leq j \leq M} Y_j] \leq \sigma \sqrt{\log M}.
  \end{equation*}
\end{theorem}

\subsection{Some technical lemmas}
\label{sec:lemmas}

In this section we state some technical Lemmas, used in the proof of
Theorem~\ref{thm:aggregation}.

\subsubsection*{Notations}

Given a sample $(Z_i)_{i=1}^n$, we set the random empirical measure
$P_n := n^{-1} \sum_{i=1}^n \delta_{Z_i}$. For any function $f$ define
$(P - P_n)(f) := n^{-1} \sum_{i=1}^n f(Z_i) - \E f(Z)$ and for a class
of functions $F$, define $\norm{P - P_n}_F := \sup_{f \in F} |(P -
P_n)(f)|$. In all what follows, we denote by $c$ an absolute positive
constant, that can vary from place to place. Its dependence on the
parameters of the setting is specified in place.

\begin{proof}[Proof of Lemma~\ref{lem:complexity}] 
  We first start with lemma 4.6 of \cite{Mendelson08regularizationin}
  to obtain
  \begin{equation}
    \label{eq:Sum-Peeling-Alpha}
    \E\norm{P-P_n}_{V_{r,\lambda}} \leq
    2 \sum_{i=0}^\infty 2^{-i} \E \norm{P-P_n}_{\cL_{r,2^{i+1}\lambda}}
  \end{equation}
  where $\cL_{r,2^{i+1}\lambda}:=\{\cL_{r,f}:f\in\cF_r,\E\cL_{r,f}\leq
  2^{i+1}\lambda\}$. Let $i\in\mathbb{N}$. Using the Gin\'{e}-Zinn
  symmetrization argument, see \cite{gz:84}, we have
  \begin{equation*}
    \E\norm{P - P_n}_{\cL_{r,2^{i+1}\lambda}} \leq
    \frac{2}{n} \E_{(X, Y)} \E_\epsilon \Big[ \sup_{\cL \in
      \cL_{r,2^{i+1}\lambda}} \Big| \sum_{i=1}^n \epsilon_i \cL(X_i,
    Y_i) \Big| \Big],
  \end{equation*}
  where $(\epsilon_i)$ is a sequence of i.i.d. Rademacher variables.
  Recall that there is (see \cite{ledoux_talagrand91}) an absolute
  constant $c_g$ such that for any $T \subset \mathbb R^n$, we have
  \begin{equation}
    \label{eq:rademacher_gaussian}
    \E_\epsilon \Big[ \sup_{t \in T} \Big| \sum_{i=1}^n \epsilon_i t_i
    \Big| \Big] \leq c_g \E_g \Big[ \sup_{t \in T} \Big| \sum_{i=1}^n
    g_i t_i \Big| \Big],
  \end{equation}
  where $(g_i)$ are i.i.d. standard normal. So, we have
  \begin{equation*}
    \E \norm{P - P_n}_{\cL_{r,2^{i+1}\lambda}} \leq \frac{2 c_g}{n} \E_{(X,
      Y)} \E_g \Big[ \sup_{\cL \in \cL_{r,2^{i+1}\lambda}} \Big| \sum_{i=1}^n
    g_i \cL(X_i, Y_i)  \Big| \Big].
  \end{equation*}
  Consider the Gaussian process $f \to Z_f := \sum_{i=1}^n g_i {\cal
    L}_f(X_i,Y_i)$ indexed by $\cF_{r, 2^{i+1} \lambda} := \{ f \in
  \cF_r : \E \cL_{r,f} \leq 2^{i+1} \lambda \}$. For every $f, f'
  \in\cF_{r,2^{i+1}\lambda}$, we have (conditionally to the
  observations)
  \begin{align*}
    \E_g |Z_f - Z_{f'}|^2 \leq 4 ( \norm{Y}_\infty + r )^2 \E_g |
    Z^\prime_f -Z^\prime_{f'} |^2,
  \end{align*}
  where $Z_f^\prime := \sum_{i=1}^n g_i (f(X_i) - f_r^*(X_i))$ and
  $f_r^* \in \argmin_{f \in \cF_r} R(f)$. Using the convexity of
  $\cF_r$, it is easy to get $\E[\mathcal L_{r,f}] \geq \norm{f -
    f_r^*}^2,\forall f\in\cF_r$. Define
  \begin{equation*}
    B_{r, 2^{i+1}\lambda} := \{ f - f_r^* : f \in \cF_r, \norm{f -
      f_r^*} \leq \sqrt{2^{i+1} \lambda} \}.
  \end{equation*}
  Using Slepian's Lemma (see, e.g.
  \cite{ledoux_talagrand91,DudleyBook99}), we have:
  \begin{equation*}
    \E\norm{P - P_n}_{\cL_{r, 2^{i+1}\lambda}} \leq
    c_Y (r + 1) E,
  \end{equation*}
  where we put
  \begin{equation*}
    E := \frac{1}{n} \E_{X} \E_g \Big[ \sup_{f \in
      B_{r, 2^{i+1}\lambda}} \Big| \sum_{i=1}^n g_i f(X_i) \Big|
    \Big].
  \end{equation*}
  Moreover, using Dudley's entropy integral argument (again,
  see \cite{ledoux_talagrand91, DudleyBook99, massart03}), and
  Assumption~\ref{ass:entropy}, we have
  \begin{align*}
    E &\leq \frac{12}{\sqrt n} \E_X \int_0^{\Delta} \sqrt{N(B_{r,
        2^{i+1} \lambda}, \norm{\cdot}_n, t )} dt \\
    &\leq \frac{12\sqrt{c}}{\sqrt n} \E_X \int_0^{\Delta} \Big(
    \frac{r}{t} \Big)^{\beta / 2} dt = \frac{c_\beta}{\sqrt n}
    r^{\beta / 2}
    \E_X[ \Delta^{1 - \beta / 2}] \\
    &\leq \frac{c_\beta}{\sqrt n} r^{\beta / 2} (\E_X[ \Delta^2])^{(1
      - \beta / 2) / 2},
  \end{align*}
  where $c_\beta = 12\sqrt{c} / (1 - \beta/2)$ and $\Delta :=
  \diam(B_{r,2^{i+1} \lambda}, \norm{\cdot}_n)$. But, one has, if
  $B_{r,2^{i+1} \lambda}^2 := \{ f^2 : f \in B_{r,2^{i+1} \lambda}
  \}$, using a contraction argument (see \cite{ledoux_talagrand91},
  Chapter~4), with again a Gin\'e-Zinn symmetrization,
  \begin{align*}
    \E_X[ \Delta^2] &\leq \E_X \norm{P - P_n}_{B_{r, 2^{i+1}
        \lambda}^2} + 2^{i+1}\lambda \\
    &\leq 4 c_g (r + 1) E + 2^{i+1}\lambda.
  \end{align*}
  Hence, $E$ satisfies
  \begin{equation*}
    E \leq \frac{c}{\sqrt n} r^{\beta / 2} ((r + 1) E +2^{i+1}
    \lambda)^{(1 - \beta / 2) / 2},
  \end{equation*}thus
  \begin{equation*}
    \E \norm{P - P_n}_{\cL_{r, 2^{i+1}\lambda}}\leq \max
    \Big(\frac{r^2}{n^{2 / (2+\beta)}}, \frac{r^{1 + \beta/2}(2^{i+1}
      \lambda)^{1/2-\beta/4}}{\sqrt{n}}\Big).
  \end{equation*}
  Plugging the last result in the sum of
  Equation~(\ref{eq:Sum-Peeling-Alpha}) entails the result.
\end{proof}

\begin{lemma}
  \label{Lemma:finite-class}
  Define 
  \begin{equation*}
    d(F) := {\rm diam}(F, L^2(\mu)), \quad \sigma^2(F) = \sup_{f \in
      F} \E[ f(X)^2], \quad \mathcal C = \conv(F),
  \end{equation*}
  and $\mathcal L_{\cal C}(\mathcal C) = \{ (Y - f(X))^2 - (Y -
  f^{\cal C}(X))^2 : f \in \cal C \}$, where $f^{\cal C} \in
  \argmin_{g \in \cal C} R(g)$. If $\max(\norm{Y}_\infty, \sup_{f \in
    F} \norm{f}_\infty) \leq b$, we have
  \begin{align*}
    &\E \Big[ \sup_{f \in F} \frac{1}{n}\sum_{i=1}^n f^2(X_i) \Big]
    \leq c \max \Big( \sigma^2(F), \frac{b^2 \log M}{n} \Big), \text{
      and } \\
    &\E \|P_n - P\|_{\mathcal L_{\mathcal C}(\mathcal C)} \leq c b
    \sqrt{ \frac{\log M}{n} } \max \Big( b \sqrt{\frac{\log M}{n}},
    d(F) \Big).
  \end{align*}
  If $\max(\norm{\varepsilon}_{\psi_2}, \norm{\sup_{f \in F} |f(X) -
    f_0(X)|}_{\psi_2} ) \leq b$, we have
  \begin{align*}
    &\E \Big[ \sup_{f \in F} \frac{1}{n}\sum_{i=1}^n f^2(X_i) \Big]
    \leq c \max \Big( \sigma^2(F), \frac{b^2 \log M \log n}{n} \Big),
    \text{ and } \\
    &\E \|P_n - P\|_{\mathcal L_{\mathcal C}(\mathcal C)} \leq c b
    \sqrt{ \frac{\log M \log n}{n} } \max \Big( b \sqrt{\frac{\log M
        \log n}{n}}, d(F) \Big).
  \end{align*}
\end{lemma}

\begin{proof}
  First, consider the case when $\norm{ \sup_{f \in F} |f(X) - f_0(X)
    |}_{\psi_2} \leq b$. Define
  \begin{equation*}
    r^2 = \sup_{f \in F} \frac 1n \sum_{i=1}^n f(X_i)^2, 
  \end{equation*}
  and note that $\E_X( r^2 ) \leq \E_X \norm{P - P_n}_{F^2} +
  \sigma(F)^2$, where $F := \{ f^2 : f \in F \}$. Using the same
  argument as in the beginning of the proof of Lemma~1, see above, we
  have
  \begin{equation*}
    \E_X \norm{P - P_n}_{F^2} \leq \frac{c} n \E_X \E_g \Big[
    \sup_{f \in F} \Big| \sum_{i=1}^n g_i f^2(X_i) \Big| \Big].
  \end{equation*}
  The process $f \mapsto Z_{2, f} = \sum_{i=1}^n g_i f^2(X_i)$ is
  Gaussian, with intrinsic distance 
  \begin{equation*}
    \E_g |Z_{2, f} - Z_{2, f'}|^2 = \sum_{i=1}^n ( f(X_i)^2 -
    f'(X_i)^2)^2 \leq d_{n, \infty}(f, f')^2 \times 4 n r^2,
  \end{equation*}
  where $d_{n, \infty}(f, f') = \max_{i=1, \ldots, n} |f(X_i) -
  f'(X_i)|$.  So, using Dudley's entropy integral, we have
  \begin{equation*}
    \E_g \norm{P - P_n}_{F^2} \leq \frac{c}{\sqrt n}
    \int_0^{\Delta_{n, \infty}(F)} \sqrt{\log N(F, d_{n, \infty}, t)} dt
    \leq c r \Delta_{n, \infty}(F) \sqrt{\frac{\log M}{n}},
  \end{equation*}
  where $\Delta_{n, \infty}$ is the $d_{n, \infty}$-diameter of
  $F$. So, we get
  \begin{align*}
    \E_X \norm{P - P_n}_{F^2} \leq c \sqrt{\frac{\log M}{n}} \E_X[
    \Delta_{n, \infty}(F) r] \leq c \sqrt{\frac{\log M}{n}} \sqrt{
      \E_X[ \Delta_{n, \infty}^2(F) ]} \sqrt{\E_X[r^2]},
  \end{align*}
  which entails that
  \begin{equation*}
    \E_X(r^2) \leq \frac{c \log M}{n} \E_X[ \Delta_{n, \infty}^2(F) ]
    + 4 \sigma(F)^2.
  \end{equation*}
  Since $\E[Z^2] \leq 2 \norm{Z}_{\psi_2}^2$ for a subgaussian
  variable $Z$, we have, by using Pisier's inequality,
  \begin{align*}
    \E_X[ \Delta_{n, \infty}^2(F) ] &\leq 4 \norm{\max_{i=1, \ldots,
        n} \sup_{f \in F} |f(X_i) - f_0(X_i) |}_{\psi_2}^2 \\
    &\leq 4 \log(n+1) \norm{ \sup_{f \in F} |f(X) - f_0(X)
      |}_{\psi_2}^2 \\
    &\leq 4 b^2 \log(n+1),
  \end{align*}
  so we have proved that
  \begin{equation*}
    \E_X(r^2) \leq c \max\Big( b^2 \frac{\log M \log n}{n},
    \sigma(F)^2 \Big).
  \end{equation*}
  When $\norm{f}_\infty \leq b$, the proof is easier, since we can use
  the contraction principle for Rademacher process after the symmetrization
  argument:
  \begin{equation*}
    \E_X \norm{P - P_n}_{F^2} \leq \frac{2} n \E_X \E_{\epsilon} \Big[
    \sup_{f \in F} \Big| \sum_{i=1}^n \epsilon_i f^2(X_i) \Big| \Big]
    \leq \frac{8 b} n \E_X \E_{\epsilon} \Big[ \sup_{f \in F} \Big|
    \sum_{i=1}^n \epsilon_i f(X_i) \Big| \Big],
  \end{equation*}
  and one obtains from this, as previously, that
  \begin{equation*}
    \E_X(r^2) \leq c \max\Big( b^2 \frac{\log M}{n}, \sigma(F)^2
    \Big).
  \end{equation*}
  Let us turn to the part of the Lemma concerning $\E \norm{P -
    P_n}_{\cal L_{\cal C}(\cal C)}$. Recall that ${\cal C} = {\rm
    conv}(F)$ and write for short $\mathcal L_f(X, Y)= {\cal L}_{\cal
    C}(f)(X, Y) = (Y - f(X))^2 - (Y - f^{\cal C}(X))^2$ for each $f
  \in {\cal C}$, where we recall that $f^{\cal C} \in \argmin_{g \in
    \cal C} R(g)$. Using the same argument as before we have
  \begin{equation*}
    \E\norm{P - P_n}_{\mathcal L_{\mathcal C}(\mathcal C)} \leq
    \frac{c}{n} \E_{(X, Y)} \E_g \Big[ \sup_{f \in \mathcal C} \Big|
    \sum_{i=1}^n g_i \mathcal L_f(X_i, Y_i) \Big| \Big].
  \end{equation*}
  Consider the Gaussian process $f \to Z_f := \sum_{i=1}^n g_i {\cal
    L}_f(X_i,Y_i)$ indexed by ${\cal C}$. For every $f, f' \in {\cal
    C}$, the intrinsic distance of $Z_f$ satisfies
  \begin{align*}
    \E_g |Z_f - Z_{f'}|^2 &= \sum_{i=1}^n ({\cal
      L}_f(X_i,Y_i) - {\cal L}_{f'}(X_i,Y_i))^2 \\
    & \leq \max_{i=1, \ldots, n} |2 Y_i - f(X_i) - f'(X_i)|^2 \times
    \sum_{i=1}^n (f(X_i) - f'(X_i))^2 \\
    & = \max_{i=1, \ldots, n} |2 Y_i - f(X_i) - f'(X_i)|^2 \times \E_g
    | Z^\prime_f -Z^\prime_{f'} |^2,
  \end{align*}
  where $Z_f^\prime := \sum_{i=1}^n g_i (f(X_i) - f^{\cal C}(X_i))$.
  Therefore, by Slepian's Lemma, we have for every
  $(X_i,Y_i)_{i=1}^n$:
  \begin{equation*}
    \E_g \Big[ \sup_{f \in \mathcal C} Z_f \Big] \leq
    \max_{i=1, \ldots, n} \sup_{f, f' \in \mathcal C} |2 Y_i - f(X_i) -
    f'(X_i)| \times \E_g \Big[ \sup_{f \in \mathcal C} Z_f'\Big],
  \end{equation*}
  and since for every $f=\sum_{j=1}^M\alpha_jf_j\in\cC$, where $\alpha_j\geq0,\forall j=1,\ldots,M$ and $\sum\alpha_j=1$, $Z^\prime_f=\sum_{j=1}^M\alpha_jZ_{f_j}$, we have
  \begin{equation*}
    \E_g \Big[ \sup_{f \in \mathcal C} Z_f'\Big] \leq
    \E_g \Big[ \sup_{f \in F} Z_f' \Big].
  \end{equation*}
  Moreover, we have
  using Dudley's entropy integral argument,
  \begin{align*}
    \frac{1}{n} \E_g \Big[ \sup_{f \in F} Z_f' \Big] &\leq
    \frac{c}{\sqrt n} \int_0^{\Delta_n(F')} \sqrt{N(F, \norm{\cdot}_n,
      t )} dt \leq c \sqrt{ \frac{\log M}{n}} r',
  \end{align*}
  where  $F' := \{ f - f^{\cal C} : f \in F \}$ and $\Delta_n(F') := \diam(F', \norm{\cdot}_n)$ and
  \begin{equation*}
    r'^2 := \sup_{f \in F'} \frac{1}{n} \sum_{i=1}^n f(X_i)^2.
  \end{equation*}
  On the other hand, we can prove, using Pisier's inequality for $\psi_1$ random variables and the fact that $\norm{U^2}_{\psi_1}=\norm{U}_{\psi_2}^2$ for every random variable $U$, that
  \begin{align}
    \label{eq:maj_moment2_psi2}
    &\sqrt{ \E \Big[ \max_{i=1, \ldots, n} \sup_{f, f' \in \mathcal C}
      |2 Y_i - f(X_i) - f'(X_i)|^2 \Big] } \\
    \nonumber &\leq 2 \sqrt{2 \log(n+1)} (\norm{\varepsilon}_{\psi_2}
    + \norm{\sup_{f \in F} |f(X) - f_0(X)|}_{\psi_2} ).
  \end{align}
  So, we finally obtain
  \begin{align*}
    \E\norm{P - P_n}_{\mathcal L_{\mathcal C}(\mathcal C)} \leq c
    \sqrt{\frac{\log n \log M}{n}} \sqrt{\E(r'^2)},
  \end{align*}
  and the conclusion follows from the first part of the Lemma, since
  $\sigma(F') \leq d(F)$. The case $\max(\norm{Y}_\infty, \sup_{f \in
    F} \norm{f}_\infty) \leq b$ is easier and follows from the fact
  that the left hand side of~\eqref{eq:maj_moment2_psi2} is smaller
  than $4b$.
\end{proof}

Lemma~\ref{Lemma:finite-class} combined with Theorem
\ref{thm:adamczak} leads to the following corollary.

\begin{corollary}
  \label{cor:high-prob-est} 
  Let $d(F) = \diam(F, L^2(P_X))$, $\mathcal C := \conv(F)$ and
  $\mathcal L_f(X, Y) = (Y - f(X))^2 - (Y - f^{\cal C}(X))^2$. If
  $\max(\norm{\varepsilon}_{\psi_2}, \norm{\sup_{f \in F} |f(X) -
    f_0(X)|}_{\psi_2} ) \leq b$ we have, with probability larger than
  $1 - 4 e^{-x}$, that for every $f \in \mathcal C$:
  \begin{align*}
    \Big| \frac{1}{n} \sum_{i=1}^n {\cal L}_f&(X_i, Y_i) -
    \E{\cal L}_f(X,Y) \Big| \\
    &\leq c (\sigma_{\varepsilon} + b) \sqrt{ \frac{(\log M + x) \log
        n}{n}} \max\Big( b \sqrt{\frac{(\log M + x) \log n}{n}} , d(F)
    \Big).
  \end{align*}
  If $\max(\norm{Y}_\infty, \sup_{f \in F} \norm{f}_\infty) \leq b$,
  we have, with probability larger than $1 - 4e^{-x}$, that for every
  $f \in \mathcal C$:
  \begin{align*}
    \Big| \frac{1}{n} \sum_{i=1}^n {\cal L}_f(X_i, Y_i) - \E{\cal
      L}_f(X,Y) \Big| \leq c b \sqrt{ \frac{\log M + x}{n}} \max\Big(
    b \sqrt{\frac{\log M + x}{n}} , d(F) \Big).
  \end{align*}
\end{corollary}

\begin{proof}
  We apply Theorem~\ref{thm:adamczak} to the process
  \begin{equation*}
    Z := \sup_{f \in {\cal C}} \Big| \frac{1}{n}\sum_{i=1}^n {\cal
      L}_{f}(X_i, Y_i) - \E{\cal L}_f(X,Y) \Big|,
  \end{equation*}
  to obtain that with a probability larger than $1 - 4 e^{-x}$:
  \begin{equation*}
    Z \leq c \Big(\E Z + \sigma(\mathcal C) \sqrt \frac{x}{n} +
    b_n(\mathcal C) \frac x n\Big),
  \end{equation*}
  where
  \begin{align*}
    \sigma(\mathcal C)^2 &= \sup_{f \in \mathcal C} \E[ \mathcal
    L_{f}(X, Y)^2 ], \text{ and }  \\
    b_n(\mathcal C) &= \big\| \max_{i=1, \ldots, n} \sup_{f \in
      \mathcal C} | \mathcal L_f(X_i, Y_i) - \E[\mathcal L_f(X, Y)] |
    \big\|_{\psi_1}.
  \end{align*}
  Since $\mathcal L_f(X, Y) = 2 \varepsilon (f^{\cal C}(X) - f(X)) +
  (f^{\cal C}(X) - f(X)) (2 f_0(X) - f(X) - f^{\cal C}(X))$, we have
  using Assumption~\ref{ass:model}:
  \begin{align*}
    \E[\mathcal L_f(X, Y)^2] &\leq 4 \sigma_{\varepsilon}^2 \norm{f -
      f^{\cal C}}^2_{L^2(P_X)} \\
    &+ 4 \sqrt{\E [(f^{\cal C}(X) - f(X))^4]} \sqrt{ \E[( 2 f_0(X) -
      f(X) - f^{\cal C}(X))^4]}.
  \end{align*}
  If $U_f := (f^{\cC}(X) - f(X))^2$ we have $\norm{U_f}_{\psi_1}=\norm{f^{\cC}-f}_{\psi_2}^2 \leq
  (2 b)^2$ for any $f \in \cal C$, so using the $\psi_1$ version of
  Bernstein's inequality (see \cite{vdVW:96}), we have that $\P( |U_f
  - \E(U_f)| \geq m \norm{U_f}_{\psi_1}) \leq 2 \exp(-c \min(m, m^2) )
  = 2 \exp(- c m)$ for any $m \in \mathbb N - \{ 0 \}$. But for such a
  random variable, one has $\E(U_f^p)^{1/p} \leq c_p \E(U_f)$ for any
  $p > 1$ (cf. \cite{MR2075996}). So, in particular for $p=2$, we derive
  \begin{equation*}
    \sqrt{\E [(f^{\cal C}(X) - f(X))^4]} \leq c \norm{f - f^{\cal
        C}}_{L^2(P_X)}^2.
  \end{equation*}
  Moreover, since $\E(Z^4) \leq 16 \norm{Z}_{\psi_2}^4$, we have
  \begin{equation*}
    \sqrt{ \E[( 2 f_0(X) - f(X) - f^{\cal C}(X))^4]} \leq 8 b^2.
  \end{equation*}
  So, we can conclude that
  \begin{equation*}
    \sigma(\mathcal C)^2 \leq (4\sigma_{\varepsilon}^2 + 8 c b^2) d(F).
  \end{equation*}
  Since $\E(Z) \leq \norm{Z}_{\psi_1}$, we have $b_n(\mathcal C) \leq
  2 \log(n+1) \norm{\sup_{f \in \cal C} |\mathcal L_f(X, Y)|
  }_{\psi_1}$. Moreover, a straightforward calculation gives $\mathcal
  L_f(X, Y) \leq \varepsilon^2 + (f^{\cal C}(X) - f_0(X))^2 + 3(f(X) -
  f_0(X))^2$, so
  \begin{equation*}
    b_n(\mathcal C) \leq 10 \log(n + 1) b^2.
  \end{equation*}
  Putting all this together, and using Lemma~\ref{Lemma:finite-class},
  we arrive at 
  \begin{equation*}
    Z \leq c ( \sigma_{\varepsilon} + b) \sqrt{ \frac{(\log M + x)
        \log n}{n}} \max\Big( b \sqrt{\frac{(\log M + x)
        \log n}{n}} , d(F) \Big),
  \end{equation*}
  with probability larger than $1 - 4 e^{-x}$ for any $x > 0$. In the
  bounded case where $\max(\norm{Y}_\infty, \sup_{f \in F}
  \norm{f}_\infty) \leq b$, the proof is easier, and one can use the
  original Talagrand's concentration inequality.
\end{proof}

\begin{lemma}
  \label{lemma:finite-F-loss}
  Let $\mathcal L_f(X, Y) = (Y - f(X))^2 - (Y - f^F(X))^2$. If we have
  $\max(\norm{\varepsilon}_{\psi_2}, \norm{\sup_{f \in F} |f(X) -
    f_0(X)|}_{\psi_2} ) \leq b$ we have, with probability larger than
  $1 - 4 e^{-x}$, that for every $f \in F$:
  \begin{align*}
    \Big| \frac{1}{n} \sum_{i=1}^n {\cal L}_f&(X_i, Y_i) -
    \E{\cal L}_f(X,Y) \Big| \\
    &\leq c (\sigma_{\varepsilon} + b) \sqrt{ \frac{(\log M + x) \log
        n}{n}} \max\Big( b \sqrt{\frac{(\log M + x) \log n}{n}} ,
    \norm{f - f^F} \Big).
  \end{align*}
  Also, with probability at least $1 - 4 e^{-x}$, we have for every
  $f,g \in F$:
  \begin{align*}
    \big| \norm{f - g}_{n}^2 &- \norm{f - g}^2 \big| \\
    &\leq c b \sqrt{ \frac{(\log M + x) \log n}{n}} \max\Big( b
    \sqrt{\frac{(\log M + x) \log n}{n}} , \norm{f - g} \Big).
  \end{align*}
  When $\max(\norm{Y}_\infty, \sup_{f \in F} \norm{f}_\infty) \leq b$,
  we have, with probability larger than $1 - 2 e^{-x}$, that for every
  $f \in F$:
  \begin{align*}
    \Big| \frac{1}{n} \sum_{i=1}^n {\cal L}_f(X_i, Y_i) - \E{\cal
      L}_f(X,Y) \Big| \leq c b \sqrt{ \frac{\log M + x}{n}} \max\Big(
    b \sqrt{\frac{\log M + x}{n}} , \norm{f - f^F} \Big),
  \end{align*}
  and with probability at least $1 - 2 e^{-x}$, that for every $f,g
  \in F$:
  \begin{align*}
    \big| \norm{f - g}_{n}^2 - \norm{f - g}^2 \big|
    \leq c b \sqrt{ \frac{\log M + x}{n}} \max\Big(
    b \sqrt{\frac{\log M + x}{n}} , \norm{f - g} \Big).
  \end{align*}
\end{lemma}

\begin{proof}[Proof of Lemma~\ref{lemma:finite-F-loss}]
  The proof uses exactly the same arguments as that of
  Lemma~\ref{Lemma:finite-class} and
  Corollary~\ref{cor:high-prob-est}, and thus is omitted.
\end{proof}

\begin{lemma}
  \label{lemma:prop-random-set} 
  Let $\hat{F}_1$ be given by~\eqref{eq:random-set} and recall that
  $f^F \in \argmin_{f \in F} R(f)$ and let $d(\hat{F}_1) =
  \diam(\hat{F}_1, L_2(P_X))$. Assume that
  \begin{equation*}
    \max(\norm{\varepsilon}_{\psi_2}, \norm{\sup_{f \in F} |f(X) -
      f_0(X)|}_{\psi_2} ) \leq b.
  \end{equation*}
  Then, with probability at least $1 - 4 \exp(-x)$, we have $f^F \in
  \hat{F}_1$, and any function $f \in \hat{F}_1$ satisfies
  \begin{equation*}
    R(f) \leq R(f^F) + c(\sigma_\varepsilon + b) \sqrt{\frac{(\log M +
        x) \log n}{n}} \max \Big( b \sqrt{\frac{(\log M + x) \log
        n}{n}}, d(\hat F_1) \Big).
  \end{equation*}
  If $\max(\norm{Y}_\infty, \sup_{f \in F} \norm{f}_\infty) \leq b$,
  we have with probability at least $1 - 2 \exp(-x)$ that $f^F \in
  \hat{F}_1$, and any function $f \in \hat{F}_1$ satisfies
  \begin{equation*}
    R(f) \leq R(f^F) + c b \sqrt{\frac{\log M + x}{n}} \max \Big( b
    \sqrt{\frac{\log M + x}{n}}, d(\hat F_1) \Big).
  \end{equation*}
\end{lemma}

\begin{proof}
  The proof follows the lines of the proof of Lemma~4.4 in
  \cite{LM06}, together with Lemma~\ref{lemma:finite-F-loss}, so we
  don't reproduce it here.
\end{proof}

\section{Function spaces}
\label{sec:appendix_approximation}

In this section we give precise definitions of the spaces of functions
considered in the paper, and give useful related results. The
definitions and results presented here can be found
in~\cite{triebel06}, in particular in Chapter~5 which is about
anisotropic spaces, anisotropic multiresolutions, and entropy numbers
of the embeddings of such spaces (see Section~5.3.3) that we use in
particular to derive condition $(C_\beta)$, for the anisotropic Besov
space, see Section~\ref{sec:pena_least_squares}.



Let $\{ e_1, \ldots, e_d \}$ be the canonical basis of $\mathbb R^d$
and $\bs s = (s_1, \ldots, s_d)$ with $s_i > 0$ be a vector of
directional smoothness, where $s_i$ corresponds to the smoothness in
direction $e_i$. Let us fix $1 \leq p, q \leq \infty$. If $f : \mathbb
R^d \rightarrow \mathbb R$, we define $\Delta_h^k f$ as the
\emph{difference} of order $k \geq 1$ and step $h \in \mathbb R^d$,
given by $\Delta_h^1 f(x) = f(x + h) - f(x)$ and $\Delta_h^k f(x) =
\Delta_h^1(\Delta_h^{k-1}f)(x)$ for any $x \in \mathbb R^d$. 

\begin{definition}
  \label{def:besov}
  We say that $f \in L^p(\mathbb R^d)$ belongs to the anisotropic
  Besov space $B_{p, q}^{\bs s}(\mathbb R^d)$ if the semi-norm
  \begin{equation*}
    |f|_{B_{p, q}^{\bs s}(\mathbb R^d)} := \sum_{i=1}^d \Big(
    \int_0^1 (t^{-s_i} \norm{\Delta_{t e_i}^{k_i} f}_{p})^q
    \frac{dt}{t} \Big)^{1/q}
  \end{equation*}
  is finite (with the usual modifications when $p = \infty$ or $q =
  \infty$).
\end{definition}
We know that the norms
\begin{equation*}
  \norm{f}_{B_{p, q}^{\bs s}} := \norm{f}_p + |f|_{B_{p, q}^{\bs s}}
\end{equation*}
are equivalent for any choice of $k_i > s_i$. An equivalent definition
of the seminorm can be given using the directional differences and the
anisotropic distance, see Theorem~5.8 in~\cite{triebel06}.



Several explicit particular cases for the space $B_{p, q}^{\bs s}$ are
of interest. If $\bs s = (s, \ldots, s)$ for some $s > 0$, then $B_{p,
  q}^{\bs s}$ is the standard isotropic Besov space. When $p = q = 2$
and $s = (s_1, \ldots, s_d)$ has integer coordinates, $B_{2, 2}^{\bs
  s}$ is the anisotropic Sobolev space
\begin{equation*}
  B_{2, 2}^{\bs s} = W_2^{\bs s} = \Big\{ f \in L^2 : \sum_{i=1}^d
  \Big\| \frac{\partial^{s_i} f}{\partial x_i^{s_i}} \Big\|_2 < \infty
  \Big\}.
\end{equation*}
If $\bs s$ has non-integer coordinates, then $B_{2, 2}^{\bs s}$ is the
anisotropic Bessel-potential space
\begin{equation*}
  H^{\bs s} = \Big\{ f \in L^2 : \sum_{i=1}^d \Big\| (1 +
  |\xi_i|^2)^{s_i/2} \hat f(\xi) \Big\|_2 < \infty \Big\}.
\end{equation*}

As we mentioned below, Assumption~\ref{ass:entropy} is satisfied for
barely all smoothness spaces considered in nonparametric
literature. In particular, if $\mathcal F = B_{p,q}^{\bs s}$ is the
anisotropic Besov space defined above, $(C_\beta)$ is satisfied: it is
a consequence of a more general Theorem (see Theorem~5.30 in
\cite{triebel06}) concerning the entropy numbers of embeddings (see
Definition~1.87 in \cite{triebel06}). Here, we only give a simplified
version of this Theorem, which is sufficient to derive $(C_\beta)$ for
$B_{p,q}^{\bs s}$. Indeed, if one takes $\bs s_0 = \bs s$, $p_0 = p$,
$q_0 = q$ and $\bs s_1 = 0$, $p_0 = \infty$, $q_0 = \infty$ in
Theorem~5.30 from \cite{triebel06}, we obtain the following
\begin{theorem}
  \label{thm:anisotropic_entropy}
  Let $1 \leq p, q \leq \infty$ and $\bs s = (s_1, \ldots, s_d)$ where
  $s_i > 0$\textup, and let $\bs {\bar s}$ be the harmonic mean of
  $\bs s$ \textup(see~\eqref{eq:harmonic_mean}\textup). Whenever $\bs
  {\bar s} > d / p$\textup, we have
  \begin{equation*}
    B_{p, q}^{\bs s} \subset C(\mathbb R^d),
  \end{equation*}
  where $C(\mathbb R^d)$ is the set of continuous functions on
  $\mathbb R^d$\textup, and for any $\delta > 0$\textup, the sup-norm
  entropy of the unit ball of the anisotropic Besov space\textup,
  namely the set
  \begin{equation*}
    U_{p, q}^{\bs s} := \{ f \in B_{p, q}^{\bs s} :
    |f|_{B_{p,q}^{\bs s}} \leq 1 \}
  \end{equation*}
  satisfies
  \begin{equation}
    H_\infty(\delta, U_{p, q}^{\bs s}) \leq D \delta^{-\bs {\bar s} / d},
  \end{equation}
  where $D > 0$ is a constant independent of $\delta$.
\end{theorem}

For the isotropic Sobolev space, Theorem~\ref{thm:anisotropic_entropy}
was obtained in the key paper~\cite{birman_solomjak67} (see
Theorem~5.2 herein), and for the isotropic Besov space, it can be
found, among others, in~\cite{birge_massart00}
and~\cite{kerk_picard_replicant_03}.

\begin{remark}
  A more constructive computation of the entropy of anisotropic Besov
  spaces can be done using the replicant coding approach, which is
  done for Besov bodies in~\cite{kerk_picard_replicant_03}. Using this
  approach together with an anisotropic multiresolution analysis based
  on compactly supported wavelets or atoms, see Section~5.2
  in~\cite{triebel06}, we can obtain a direct computation of the
  entropy. The idea is to do a quantization of the wavelet
  coefficients, and then to code them using a replication of their
  binary representation, and to use 01 as a separator (so that the
  coding is injective). A lower bound for the entropy can be obtained
  as an elegant consequence of Hoeffding's deviation inequality for
  sums of i.i.d. variables and a combinatorial lemma.
\end{remark}

\par

\bibliographystyle{ims}


\begin{thebibliography}{39}
\expandafter\ifx\csname natexlab\endcsname\relax\def\natexlab#1{#1}\fi
\expandafter\ifx\csname url\endcsname\relax
  \def\url#1{\texttt{#1}}\fi
\expandafter\ifx\csname urlprefix\endcsname\relax\def\urlprefix{URL }\fi
\providecommand{\eprint}[2][]{\url{#2}}

\bibitem[{Adamczak(2008)}]{MR2424985}
\textsc{Adamczak, R.} (2008).
\newblock A tail inequality for suprema of unbounded empirical processes with
  applications to {M}arkov chains.
\newblock \textit{Electron. J. Probab.}, \textbf{13} no. 34, 1000--1034.

\bibitem[{Audibert(2009)}]{audibert-2009-37}
\textsc{Audibert, J.-Y.} (2009).
\newblock Fast learning rates in statistical inference through aggregation.
\newblock \textit{Ann. Statist.}, \textbf{37} 1591.
\newblock \urlprefix\url{doi:10.1214/08-AOS623}.

\bibitem[{Birg{\'e} and Massart(2000)}]{birge_massart00}
\textsc{Birg{\'e}, L.} and \textsc{Massart, P.} (2000).
\newblock An adaptive compression algorithm in {B}esov spaces.
\newblock \textit{Constr. Approx.}, \textbf{16} 1--36.

\bibitem[{Birman and Solomjak(1967)}]{birman_solomjak67}
\textsc{Birman, M.~{\v{S}}.} and \textsc{Solomjak, M.~Z.} (1967).
\newblock Piecewise polynomial approximations of functions of classes
  {$W_p^{\alpha}$}.
\newblock \textit{Mat. Sb. (N.S.)}, \textbf{73 (115)} 331--355.

\bibitem[{Catoni(2001)}]{catbook:01}
\textsc{Catoni, O.} (2001).
\newblock \textit{Statistical Learning Theory and Stochastic Optimization}.
\newblock Ecole d'{\'e}t{\'e} de Probabilit{\'e}s de Saint-Flour 2001, Lecture
  Notes in Mathematics, Springer, N.Y.

\bibitem[{Cucker and Smale(2002)}]{cucker_smale02}
\textsc{Cucker, F.} and \textsc{Smale, S.} (2002).
\newblock On the mathematical foundations of learning.
\newblock \textit{Bull. Amer. Math. Soc. (N.S.)}, \textbf{39} 1--49
  (electronic).

\bibitem[{Dalalyan and Tsybakov(2007)}]{dalalyan_tsybakov07}
\textsc{Dalalyan, A.~S.} and \textsc{Tsybakov, A.~B.} (2007).
\newblock Aggregation by exponential weighting and sharp oracle inequalities.
\newblock In \textit{COLT}. 97--111.

\bibitem[{Dudley(1999)}]{DudleyBook99}
\textsc{Dudley, R.~M.} (1999).
\newblock \textit{Uniform central limit theorems}, vol.~63 of \textit{Cambridge
  Studies in Advanced Mathematics}.
\newblock Cambridge University Press, Cambridge.

\bibitem[{Efron et~al.(2004)Efron, Hastie, Johnstone and
  Tibshirani}]{MR2060166}
\textsc{Efron, B.}, \textsc{Hastie, T.}, \textsc{Johnstone, I.} and
  \textsc{Tibshirani, R.} (2004).
\newblock Least angle regression.
\newblock \textit{Ann. Statist.}, \textbf{32} 407--499.
\newblock With discussion, and a rejoinder by the authors.

\bibitem[{Einmahl and Mason(1996)}]{EM:96}
\textsc{Einmahl, U.} and \textsc{Mason, D.~M.} (1996).
\newblock Some universal results on the behavior of increments of partial sums.
\newblock \textit{Ann. Probab.}, \textbf{24} 1388--1407.

\bibitem[{Gin{\'e} and Zinn(1984)}]{gz:84}
\textsc{Gin{\'e}, E.} and \textsc{Zinn, J.} (1984).
\newblock Some limit theorems for empirical processes.
\newblock \textit{Ann. Probab.}, \textbf{12} 929--998.

\bibitem[{Gy{\"o}rfi et~al.(2002)Gy{\"o}rfi, Kohler, Krzy{\.z}ak and
  Walk}]{kohler02}
\textsc{Gy{\"o}rfi, L.}, \textsc{Kohler, M.}, \textsc{Krzy{\.z}ak, A.} and
  \textsc{Walk, H.} (2002).
\newblock \textit{A distribution-free theory of nonparametric regression}.
\newblock Springer Series in Statistics, Springer-Verlag, New York.

\bibitem[{Hoffmann and Lepski(2002)}]{hoffmann_lepski02}
\textsc{Hoffmann, M.} and \textsc{Lepski, O.~V.} (2002).
\newblock Random rates in anisotropic regression.
\newblock \textit{The Annals of Statistics}, \textbf{30} 325--396.

\bibitem[{Juditsky et~al.(2008)Juditsky, Rigollet and Tsybakov}]{MR2458184}
\textsc{Juditsky, A.}, \textsc{Rigollet, P.} and \textsc{Tsybakov, A.~B.}
  (2008).
\newblock Learning by mirror averaging.
\newblock \textit{Ann. Statist.}, \textbf{36} 2183--2206.
\newblock
  \urlprefix\url{https://acces-distant.upmc.fr:443/http/dx.doi.org/10.1214/07-%
AOS546}.

\bibitem[{Juditsky et~al.(2005)Juditsky, Nazin, Tsybakov and
  Vayatis}]{juditsky_nazin05}
\textsc{Juditsky, A.~B.}, \textsc{Nazin, A.~V.}, \textsc{Tsybakov, A.~B.} and
  \textsc{Vayatis, N.} (2005).
\newblock Recursive aggregation of estimators by the mirror descent method with
  averaging.
\newblock \textit{Problemy Peredachi Informatsii}, \textbf{41} 78--96.

\bibitem[{Kerkyacharian et~al.(2001)Kerkyacharian, Lepski and
  Picard}]{kerk_lepski_picard01}
\textsc{Kerkyacharian, G.}, \textsc{Lepski, O.} and \textsc{Picard, D.} (2001).
\newblock Nonlinear estimation in anisotropic multi-index denoising.
\newblock \textit{Probab. Theory Related Fields}, \textbf{121} 137--170.

\bibitem[{Kerkyacharian et~al.(2007)Kerkyacharian, Lepski and
  Picard}]{kerk_lepski_picard07}
\textsc{Kerkyacharian, G.}, \textsc{Lepski, O.} and \textsc{Picard, D.} (2007).
\newblock Nonlinear estimation in anisotropic multiindex denoising. {S}parse
  case.
\newblock \textit{Teor. Veroyatn. Primen.}, \textbf{52} 150--171.

\bibitem[{Kerkyacharian and Picard(2003)}]{kerk_picard_replicant_03}
\textsc{Kerkyacharian, G.} and \textsc{Picard, D.} (2003).
\newblock Replicant compression coding in {B}esov spaces.
\newblock \textit{ESAIM Probab. Stat.}, \textbf{7} 239--250 (electronic).

\bibitem[{Lecu{\'e}(2006)}]{LecJMLR:06}
\textsc{Lecu{\'e}, G.} (2006).
\newblock Lower bounds and aggregation in density estimation.
\newblock \textit{J. Mach. Learn. Res.}, \textbf{7} 971--981.

\bibitem[{Lecu{\'e} and Mendelson(2009{\natexlab{a}})}]{LM06}
\textsc{Lecu{\'e}, G.} and \textsc{Mendelson, S.} (2009{\natexlab{a}}).
\newblock Aggregation via empirical risk minimization.
\newblock \textit{Probab. Theory Related Fields}, \textbf{145} 591--613.
\newblock
  \urlprefix\url{https://acces-distant.upmc.fr:443/http/dx.doi.org/10.1007/s00%
440-008-0180-8}.

\bibitem[{Lecu{\'e} and Mendelson(2009{\natexlab{b}})}]{LM2}
\textsc{Lecu{\'e}, G.} and \textsc{Mendelson, S.} (2009{\natexlab{b}}).
\newblock Sharper lower bounds on the performance of the empirical risk
  minimization algorithm.
\newblock \textit{To appear in Bernoulli}.

\bibitem[{Ledoux and Talagrand(1991)}]{ledoux_talagrand91}
\textsc{Ledoux, M.} and \textsc{Talagrand, M.} (1991).
\newblock \textit{Probability in {B}anach spaces}, vol.~23 of
  \textit{Ergebnisse der Mathematik und ihrer Grenzgebiete (3) [Results in
  Mathematics and Related Areas (3)]}.
\newblock Springer-Verlag, Berlin.
\newblock Isoperimetry and processes.

\bibitem[{Lee et~al.(1996)Lee, Bartlett and Williamson}]{lbw:96}
\textsc{Lee, W.~S.}, \textsc{Bartlett, P.~L.} and \textsc{Williamson, R.~C.}
  (1996).
\newblock The importance of convexity in learning with squared loss.
\newblock In \textit{Proceedings of the Ninth Annual Conference on
  Computational Learning Theory}. ACM Press, 140--146.

\bibitem[{Leung and Barron(2006)}]{leung_barron06}
\textsc{Leung, G.} and \textsc{Barron, A.~R.} (2006).
\newblock Information theory and mixing least-squares regressions.
\newblock \textit{IEEE Trans. Inform. Theory}, \textbf{52} 3396--3410.

\bibitem[{Loustau(2009)}]{loustau09}
\textsc{Loustau, S.} (2009).
\newblock Penalized empirical risk minimization over besov spaces.
\newblock \textit{Electronic Journal of Statistics}, \textbf{3} 824--850.

\bibitem[{Massart(2007)}]{massart03}
\textsc{Massart, P.} (2007).
\newblock \textit{Concentration inequalities and model selection}, vol. 1896 of
  \textit{Lecture Notes in Mathematics}.
\newblock Springer, Berlin.
\newblock Lectures from the 33rd Summer School on Probability Theory held in
  Saint-Flour, July 6--23, 2003, With a foreword by Jean Picard.

\bibitem[{Mendelson(2004)}]{MR2075996}
\textsc{Mendelson, S.} (2004).
\newblock On the performance of kernel classes.
\newblock \textit{J. Mach. Learn. Res.}, \textbf{4} 759--771.
\newblock \urlprefix\url{http://dx.doi.org/10.1162/1532443041424337}.

\bibitem[{Mendelson(2008)}]{m:08}
\textsc{Mendelson, S.} (2008).
\newblock Lower bounds for the empirical minimization algorithm.
\newblock \textit{IEEE Trans. Inform. Theory}, \textbf{54} 3797--3803.

\bibitem[{Mendelson and Neeman(2009)}]{Mendelson08regularizationin}
\textsc{Mendelson, S.} and \textsc{Neeman, J.} (2009).
\newblock Regularization in kernel learning.
\newblock Tech. rep.
\newblock To appear in Annals of Statistics, availble at
  http://www.imstat.org/aos/.

\bibitem[{Neumann(2000)}]{neumann00}
\textsc{Neumann, M.~H.} (2000).
\newblock Multivariate wavelet thresholding in anisotropic function spaces.
\newblock \textit{Statist. Sinica}, \textbf{10} 399--431.

\bibitem[{Talagrand(1996)}]{MR1419006}
\textsc{Talagrand, M.} (1996).
\newblock New concentration inequalities in product spaces.
\newblock \textit{Invent. Math.}, \textbf{126} 505--563.

\bibitem[{Tibshirani(1996)}]{MR1379242}
\textsc{Tibshirani, R.} (1996).
\newblock Regression shrinkage and selection via the lasso.
\newblock \textit{J. Roy. Statist. Soc. Ser. B}, \textbf{58} 267--288.

\bibitem[{Triebel(2006)}]{triebel06}
\textsc{Triebel, H.} (2006).
\newblock \textit{Theory of function spaces. {III}}, vol. 100 of
  \textit{Monographs in Mathematics}.
\newblock Birkh\"auser Verlag, Basel.

\bibitem[{Tsybakov(2003{\natexlab{a}})}]{tsybakov03}
\textsc{Tsybakov, A.} (2003{\natexlab{a}}).
\newblock \textit{Introduction à l'estimation non-paramétrique}.
\newblock Springer.

\bibitem[{Tsybakov(2003{\natexlab{b}})}]{tsy:03}
\textsc{Tsybakov, A.~B.} (2003{\natexlab{b}}).
\newblock Optimal rates of aggregation.
\newblock \textit{Computational Learning Theory and Kernel Machines.
  B.Sch{\"o}lkopf and M.Warmuth, eds. Lecture Notes in Artificial
  Intelligence}, \textbf{2777} 303--313.
\newblock Springer, Heidelberg.

\bibitem[{van~de Geer(2000)}]{van_de_geer00}
\textsc{van~de Geer, S.~A.} (2000).
\newblock \textit{Applications of empirical process theory}, vol.~6 of
  \textit{Cambridge Series in Statistical and Probabilistic Mathematics}.
\newblock Cambridge University Press, Cambridge.

\bibitem[{van~der Vaart and Wellner(1996)}]{vdVW:96}
\textsc{van~der Vaart, A.~W.} and \textsc{Wellner, J.~A.} (1996).
\newblock \textit{Weak convergence and empirical processes}.
\newblock Springer Series in Statistics, Springer-Verlag, New York.
\newblock With applications to statistics.

\bibitem[{Wahba(1990)}]{wahba90}
\textsc{Wahba, G.} (1990).
\newblock \textit{Spline models for observational data}, vol.~59 of
  \textit{CBMS-NSF Regional Conference Series in Applied Mathematics}.
\newblock Society for Industrial and Applied Mathematics (SIAM), Philadelphia,
  PA.

\bibitem[{Yang(2000)}]{MR1762904}
\textsc{Yang, Y.} (2000).
\newblock Mixing strategies for density estimation.
\newblock \textit{Ann. Statist.}, \textbf{28} 75--87.
\newblock \urlprefix\url{http://dx.doi.org/10.1214/aos/1016120365}.

\end{thebibliography}

\end{document}